\documentclass[10pt,twocolumn,letterpaper]{article}

\usepackage{cvpr}

\usepackage{times}
\usepackage[utf8]{inputenc}
\usepackage[english]{babel}
\usepackage{url}
\usepackage[pagebackref=true,breaklinks=true,letterpaper=true,colorlinks,bookmarks=false]{hyperref}
\usepackage{epsfig}
\usepackage{epstopdf}
\usepackage{graphicx}
\usepackage{amsmath}
\usepackage{amsthm}
\usepackage{amssymb}
\usepackage{ mathrsfs }
\usepackage{color}
\usepackage{algpseudocode}
\usepackage{algorithm}
\usepackage{placeins}
\usepackage{rotating}
\usepackage{capt-of,etoolbox}
\usepackage{rotating}
\usepackage{cite}
\usepackage{bm}

\usepackage{ulem} % for \sout -- strikethrough
\DeclareMathOperator{\tr}{tr}
\DeclareMathOperator{\di}{diag}
\DeclareMathOperator{\ve}{vec}
\DeclareMathOperator{\prox}{prox}
\DeclareMathOperator{\diag}{diag}
\DeclareMathOperator{\sign}{sgn}

\newcommand{\Y}{\bm{\mathcal{Y}}}

\newcommand{\G}{{\mathcal{G}}}
\newcommand{\K}{{k_{nn}}}
\newcommand{\V}{\mathbb{V}}
\newcommand{\MLT}{\mathbb{MLT}}
\newcommand{\D}{D}
\newcommand{\X}{\bm{\mathcal{X}}}
\newcommand{\eS}{\bm{\mathcal{S}}}
\newcommand{\eR}{\bm{\mathcal{R}}}
\newcommand{\eM}{\bm{\mathcal{M}}}

\newcommand{\E}{\mathbb{E}}

\newcommand{\Rbb}{\mathbb{R}}

\newcommand{\kr}{\otimes}

\newcommand{\ones}{\ensuremath{\mathbf{1}}}

\newtheorem{thm}{Theorem}
\newtheorem{defn}{Definition}
\newtheorem{lemma}{Lemma}

\usepackage[pagebackref=true,breaklinks=true,letterpaper=true,colorlinks,bookmarks=false]{hyperref}

\def\cvprPaperID{0304} % *** Enter the CVPR Paper ID here

\ifcvprfinal\fi
\begin{document}

%%%%%%%%% TITLE
\title{Multilinear Low-Rank Tensors on Graphs \& Applications}

\author{Nauman Shahid$^{*}$, Francesco Grassi$^\dagger$, Pierre Vandergheynst$^{*}$\\
$^{*}$Signal Processing Laboratory (LTS2), EPFL, Switzerland, \tt\small firstname.lastname@epfl.ch \\
$^\dagger$ Politecnico di Torino, Italy, \tt\small francesco.grassi@polito.it
}

\makeatletter
\let\@oldmaketitle\@maketitle% Store \@maketitle
\renewcommand{\@maketitle}{\@oldmaketitle% Update \@maketitle to insert...
\vspace{-0.4cm}
  \centering \includegraphics[width=0.9\linewidth]{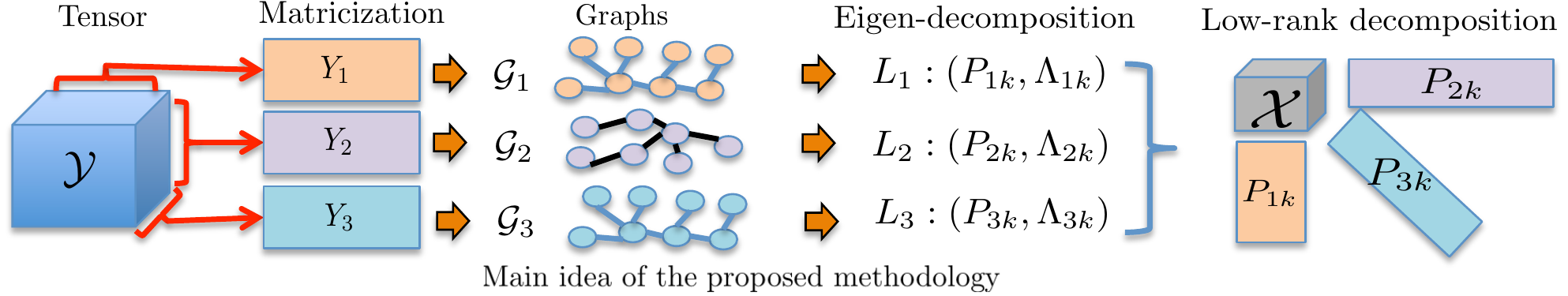}\bigskip}% ... an image
\makeatother

\maketitle
%\thispagestyle{empty}

%%%%%%%%% ABSTRACT

\begin{abstract}\vspace{-0.4cm}
We propose a new framework for the analysis of low-rank tensors which lies at the intersection of spectral graph theory and signal processing.  As a first step, we present a new graph based low-rank decomposition which approximates the classical low-rank SVD for matrices and multilinear SVD for tensors. Then, building on this novel decomposition we construct a general class of convex optimization problems  for approximately solving  low-rank tensor inverse problems, such as tensor Robust PCA. The whole framework is named as ``Multilinear Low-rank tensors on Graphs (MLRTG)''. Our theoretical analysis shows: 1) MLRTG stands on the notion of approximate stationarity of multi-dimensional signals on graphs and 2) the approximation error depends on the eigen gaps of the graphs. We demonstrate applications  for a wide variety of 4 artificial and 12 real tensor datasets, such as EEG, FMRI, BCI, surveillance videos and hyperspectral images.  Generalization of the tensor concepts to non-euclidean domain, orders of magnitude speed-up, low-memory requirement and significantly enhanced performance at low SNR are the key aspects of our framework.
   
\end{abstract}

\vspace{-0.7cm}
%%%%%%%%% BODY TEXT
\section{Introduction}
\vspace{-0.2cm}
Low-rank tensors span a wide variety of applications, like tensor compression \cite{kolda2009tensor, grasedyck2013literature}, robust PCA \cite{tomioka2010estimation, goldfarb2014robust}, completion \cite{liu2013tensor, da2013hierarchical, kressner2014low, huang2015provable} and parameter approximation in CNNs \cite{tai2015convolutional}.  However, the current literature lacks development in two main facets: 1) large scale processing and 2) generalization to non-euclidean domain (graphs \cite{shuman2013emerging}). For example, for a tensor $\Y \in \Rbb^{n \times n \times n}$, tensor robust PCA via nuclear norm  minimization \cite{tomioka2010extension} costs $\mathcal{O}(n^4)$ per iteration which is unacceptable even for $n$ as small as 100. Many application specific alternatives, such as randomization, sketching and fast optimization methods have been proposed to speed up computations \cite{tsourakakis2010mach, kang2012gigatensor, phan2013fast, huang2013fast, choi2014dfacto, huang2014distributed, bhojanapalli2015new, wang2015fast}, but they cannot be generalized for a broad range of applications.

%and tensor clustering \cite{he2005tensor, li2008discriminant, chen2014multilinear, hu2016semi}.

In this work, we answer the following question: \textit{Is it possible to 1) generalize the notion of tensors to graphs and 2) target above applications in a scalable manner?} To the best of our knowledge, little effort has been made to target the former \cite{wang2011image}, \cite{wang2012neighborhood}, \cite{zhong2014large} at the price of higher cost, however, no work has been done to tackle the two problems simultaneously. Therefore, we revisit tensors from a new perspective and develop an entirely novel, scalable and approximate framework which benefits from \textit{graphs}.  

%\vspace{-0.1cm}
 It has recently been shown for the case of 1D  \cite{2016arXiv160102522P}  and time varying signals \cite{loukas2016stationary} that the first few eigenvectors of the graph provide a smooth basis for data, the notion of graph stationarity. We generalize this concept for higher order tensors and develop a framework that encodes the tensors as a multilinear combination of few graph eigenvectors constructed from the rows of its different modes (figure on the top of this page).  This multilinear combination, which we call \textit{graph core tensor (GCT)}, is highly structured like the core tensor obtained by Multilinear SVD (MLSVD) \cite{kolda2009tensor} and can be used to solve a plethora of tensor related inverse problems in a highly scalable manner.

%\vspace{-0.1cm}
\textbf{Contributions:} In this paper we propose Multilinear low-rank tensors on graphs (MLRTG) as a novel signal model for low-rank tensors. Using this signal model, we develop an entirely novel, scalable and approximate framework for a variety of inverse problems involving tensors, such as Multilinear SVD and robust tensor PCA. Most importantly, we theoretically link the concept to joint approximate graph stationarity and characterize the approximation error in terms of the eigen gaps of the graphs. Various experiments on a wide variety of 4 artificial and 12 real benchmark datasets such as videos, face images and hyperspectral images using our algorithms demonstrate the power of our approach. The MLRTG framework is highly scalable, for example, for a tensor $\Y \in \Rbb^{n \times n \times n}$, Robust Tensor PCA on graphs scales with $\mathcal{O}(n k^2  + k^{4})$, where $k \ll n$, as opposed to $\mathcal{O}(n^4)$.

%2) Graph Core Tensor Pursuit (GCTP), an efficient and scalable algorithm to recover GCT from a noisy tensor, 3) Graph Multilinear SVD (GMLSVD), an approximate alternative to the standard MLSVD, 4) Robust Tensor PCA on graphs, a fast and approximate alternative to Tensor Robust PCA, 5) Tensor clustering algorithm and 6) a general form of inverse problems involving MLRTG, such as tensor completion and tensor compressive low-rank recovery on graphs.

\textbf{Notation:} We represent tensors with bold calligraphic letters $\Y$ and matrices with capital letters $Y$. For a tensor $\Y \in \Rbb^{n_1 \times n_2 \times n_3}$, with a multilinear rank $(r_1, r_2, r_3)$ \cite{cichocki2015tensor}, its $\mu^{th}$ matricization / flattening $Y_\mu$ is a re-arrangement such that $Y_1 \in \Rbb^{n_1 \times n_2 n_3}$. For simplicity we work with 3D tensors of same size $n$ and rank $r$ in each dimension.

\textbf{Graphs:} We specifically refer to a $\K$-nearest neighbors graph between the rows of $Y_\mu$ as $\G_\mu = (\V_\mu, \E_\mu, W_\mu)$ with vertex set $\V_\mu$, edge set $\E_\mu$ and weight matrix $W_\mu$. $W_\mu$, as defined in \cite{shuman2013emerging}, is constructed via a Gaussian kernel and the combinatorial Laplacian is given as $L_\mu = D_\mu - W_\mu$, where $D_\mu$ is the degree matrix. The eigenvalue of decomposition of $L_\mu = P_\mu \Lambda_\mu P_\mu^\top$ and we refer to the 1st $k$ eigenvectors and eigenvalues as $(P_{\mu k}, \Lambda_{\mu k})$. Throughout, we use FLANN \cite{muja2014scalable} for the  construction of $\G_\mu$ which costs $\mathcal{O}(n \log(n))$ and is parallelizable. We also assume that a fast and parallelizable framework, such as the one proposed in \cite{paratte2016fast} or \cite{si2014multi} is available for the computation of $P_{\mu k}$ which costs $\mathcal{O}(n \frac{k^2}{c})$, where $c$ is the number of processors.

\vspace{-0.2cm}
\section{Multilinear Low-Rank Tensors on Graphs}\label{sec:mlrtg}
\vspace{-0.2cm}
 A tensor $\Y^* \in \Rbb^{n \times n \times n}$ is said to be Multilinear Low-Rank on Graphs (MLRTG) if it can be encoded in terms of the lowest $k \ll n$ Laplacian eigenvectors as:
$$\ve(\Y^*) = (P_{1k} \otimes P_{2k} \otimes P_{3k})\ve(\X^*),$$
where $\ve(\cdot)$ denotes the vectorization, $\otimes$ denotes the kronecker product, $P_{\mu k} \in \Rbb^{n \times k},\forall \mu$ and $\X^* \in \Rbb^{k \times k \times k}$ is the \textit{Graph Core Tensor (GCT)}. We refer to a tensor from the set of all possible MLRTG as $\Y^* \in \MLT$. The main idea is illustrated in the figure on the first page of this paper. We call the tuple $(k,k,k)$, where $r \leq k \ll n$,  as the \textit{Graph Multilinear Rank} of $\Y^*$. In the sequel, for the simplicity of notation: 1) we work with the matricized version (along mode 1) of $\Y^*$ and 2) denote $P_{2,3 k} = P_{1k} \otimes P_{2k} \in \Rbb^{n^2 \times k^2}$. Then for $X^*_1 \in \Rbb^{k \times k^2}$, $Y^*_1 \in \Rbb^{n \times n^2} = P_{1k} X^*_1 P^\top_{2,3k}$.

\textbf{In simple words:} one can encode a low-rank tensor in terms of the low-frequency Laplacian eigenvectors. This multilinear combination is called GCT. It is highly structured like the core tensor obtained by standard Multilinear SVD (MLSVD) and can be used for a broad range of tensor based applications. Furthermore,  the fact that GCT encodes the interaction between the graph eigenvectors, renders its interpretation as a multi-dimensional graph fourier transform. 

In real applications, due to noise the tensor $\Y$ is only approximately low-rank (approximate MLRTG), so the following Lemma holds:
\begin{lemma}\label{lem:amlrtg}
For any $\Y = \Y^* + \bar{\Y} \in \Rbb^{n \times n \times n}$, where $\Y^* \in \MLT$ and $\bar{\Y}$ models the noise and errors, the $\mu^{th}$ matricization $Y_\mu$ of $\Y$ satisfies 
\begin{equation}\label{eq:mlrtg}
Y_1 = P_{1k} X^*_1 P^\top_{2,3k} + \bar{P}_{1k} \bar{X}_1 \bar{P}^\top_{2,3k},
\end{equation}
where $\bar{P}_{1 k} \in \Rbb^{n \times (n-k)}$ and $\bar{P}_{2,3 k} \in \Rbb^{n^2 \times (n-k)^2}$ denote the complement Laplacian eigenvectors (above $k$) and $\bar{X}_1 \in \Rbb^{(n-k)\times (n-k)^2}$. Furthermore, $\|\bar{X}_1\|_F \ll \|X_1\|_F$.
\end{lemma}
\begin{proof} Please refer to the proof of Theorem \ref{thm:props} in Appendix \ref{sec:proof_props}. 
\end{proof}

\begin{figure}[htbp]
    \centering
        \centering
        \includegraphics[width=0.45\textwidth]{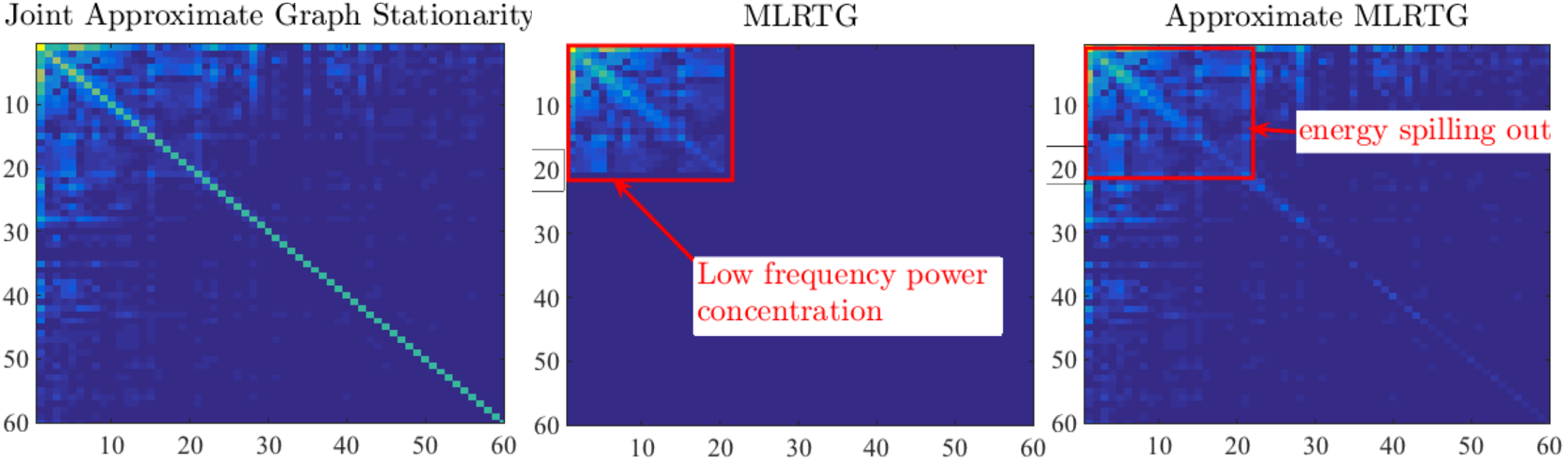}
         \caption{Illustration the properties of MLRTG in terms of an arbitrary graph spectral covariance (GSC) $\Gamma_\mu$ matrix.}
        \label{fig:stationarity}
        %\vspace{-1.2em}
    \end{figure}
    
    \begin{figure*}[htbp]
    \centering
        \centering
        \includegraphics[width=1.0\textwidth]{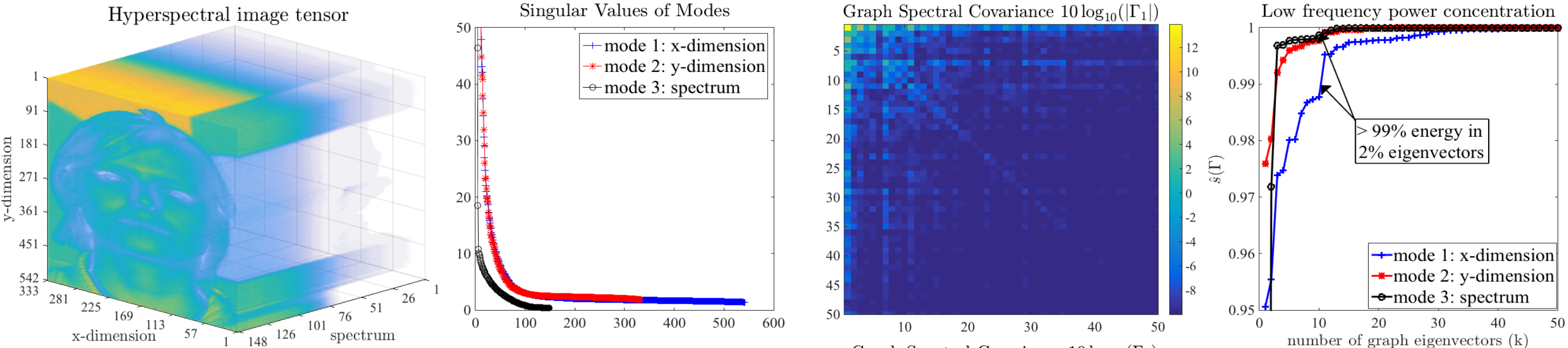}
         \caption{A hyperspectral image tensor from the Stanford database. The singular values of the modes, graph spectral covariance (GSC) and the energy concentration plot clearly show that the tensor is approximately MLRTG. }
        \label{fig:GSC}
        \vspace{-1.2em}
    \end{figure*}

Let $C_\mu \in \Rbb^{n \times n^2}$ be the covariance of $Y_\mu$, then the \textit{Graph Spectral Covariance (GSC)} $\Gamma_\mu$ is given as $\Gamma_\mu = P^\top_{\mu} C_\mu P_\mu$. For a signal that is approximately stationary on a graph, $\Gamma_\mu$ has most of its energy concentrated on the diagonal \cite{2016arXiv160102522P}. For MLRTG, we additionally require the energy to be concentrated on the top corner of $\Gamma_\mu$, which we call \textit{low-frequency energy concentration}.

\textbf{Key properties of MLRTG:} Thus, for any $\Y \in \MLT$, the GSC $\Gamma_\mu$ of each of its $\mu^{th}$ matricization $Y_\mu$  satisfies: 1) joint approximate graph stationarity, i.e, $\|\diag(\Gamma_\mu)\|^2_F/\|\Gamma_\mu\|^2_F \approx 1$ and 2) low frequency energy concentration, i.e, $\|\Gamma_\mu(1:k,1:k)\|^2_F/\|\Gamma_\mu\|^2_F \approx 1$, $\forall \mu$. 

\begin{thm}\label{thm:props} 
For any $\Y = \Y^* + \bar{\Y}$, $\Y^* \in \MLT$ if and only Lemma \ref{lem:amlrtg} and property 2 hold.
\end{thm}
\begin{proof} Please refer to Appendix \ref{sec:proof_props}. 
\end{proof}

Fig. \ref{fig:stationarity} illustrates the properties in terms of an arbitrary GSC matrix. The leftmost plot corresponds to the case of approximate stationarity (strong diagonal), the middle to the case of non-noisy MLRTG and the rightmost plot to the case of approximate MLRTG. Note that the energy spilling out of the top left submatrix due to noise results in an approximate low-rank representation.

\textbf{Examples:} Many real world datasets satisfy the approximate MLRTG assumption. Fig. \ref{fig:GSC} presents the example of a hyperspectral face tensor. The singular values for each of the modes show the low-rankness property, whereas the graph spectral covariance and the energy concentration plot (property 2) show that 99\% of the energy of the mode 1 can be expressed in terms of the top 2\% of the graph eigenvectors. Examples of FMRI and COIL20 tensor are also presented in Fig. \ref{fig:GSC_two} of Appendix \ref{sec:eg_amlrtg}.

\vspace{-0.2cm}
\section{Applications of MLRTG}\label{sec:applications}
\vspace{-0.2cm}
Any $\Y^* \in \MLT$ is the product of two important components 1)  the Laplacian eigenvectors and eigenvalues $P_{\mu k}$, $\Lambda_{\mu k}$ for each of the modes of the tensor and 2) the GCT $\X$. While the former can be pre-computed, the GCT needs to be determined via an appropriate procedure. Once determined, it can be used  directly as a low-dimensional feature of the tensor $\Y$ or employed for other useful tasks. It is therefore possible to propose a general framework for solving a broad range of tensor / matrix inverse problems, which optimize the GCT $\X$. For a general linear operator $\eM$ and its matricization $M_1$,
\begin{equation}\label{eq:ginvt}
\min_{\X} \phi(M_1(P_{1k}X_1 P^\top_{2,3 k})-Y_1) + \gamma \sum_\mu \|X_\mu\|_{*g(\Lambda_{\mu k})},
\vspace{-1.2em}
\end{equation}
where $\phi(\cdot)$ is an $l_p$ norm depending on the application under consideration and $g(\Lambda_{\mu k}) = \Lambda^{\alpha}_{\mu k}, \alpha \geq 1 $, denote the kernelized Laplacian eigenvalues as the weights for the nuclear norm minimization. Assuming the eigenvalues are sorted in ascending order, this corresponds to a higher penalization of higher singular values of $X_\mu$ which correspond to noise. Thus, the goal is to determine a graph core tensor $\X$ whose rank is minimized in all the modes. Such a nuclear norm minimization on the full tensor (without weights) has appeared in earlier works \cite{goldfarb2014robust}. However, note that in our case we lift the computational burden by minimizing only the core tensor $\X$. 

\vspace{-0.2cm}
\subsection{ Graph Core Tensor Pursuit (GCTP)}\label{sec:gctp}
\vspace{-0.2cm}
The first application corresponds to the case where one is only interested in the GCT $\X$. For a clean matricized tensor $Y_1$, it is straight-forward to determine the matricized $\X$ as $X_1 = P^\top_{1k}Y_1 P_{2,3 k}$. 

For the case of noisy $\Y$ corrupted with Gaussian noise, one seeks a robust $\X$ which is not possible without an appropriate regularization on $\X$. Hence, we propose to solve problem \ref{eq:ginvt} with Frobenius norm:
\begin{equation}\label{eq:gctp1}
\min_{\X} \|Y_1 - P_{1k }X_1 P^{\top}_{2,3 k}\|^2_F + \gamma \sum_\mu \|X_\mu\|_{*g(\Lambda_{\mu k})},
\vspace{-1.2em}
\end{equation}
 Using $Y_1 = P_{1k }\hat{X}_1 P^{\top}_{2,3 k}$ in eq. \eqref{eq:gctp1}, we get:
\begin{equation}\label{eq:gctp}
\min_{\X} \|\hat{X}_1 - X_1\|^2_F + \gamma \sum_\mu \|X_\mu\|_{*g(\Lambda_{\mu k})}.
\vspace{-1.2em}
\end{equation}
which we call as \textit{Graph Core Tensor Pursuit (GCTP)}. To solve GCTP, one just needs to apply the singular value soft-thresholding operation (Appendix \ref{sec:algo_trpca}) on each of the modes of the tensor $\X$. For $\X \in \Rbb^{k\times k \times k}$, it scales with $\mathcal{O}(k^4)$, where $k \ll n$. 

\subsection{Graph Multilinear SVD (GMLSVD)}\label{sec:gmlsvd}
\vspace{-0.2cm}
Notice that the decomposition defined by eq. \eqref{eq:mlrtg} is quite similar to the standard Mulitlinear SVD (MLSVD) \cite{kolda2009tensor}. Is it possible to define a graph based MLSVD using $\X$ ?

\textbf{MLSVD:} In standard MLSVD, one aims to decompose a tensor $\Y \in \Rbb^{n \times n \times n}$ into factors $U_\mu \in \Rbb^{n \times r}$ which are linked by a core $\eS \in \Rbb^{r \times r \times r}$. This can be attained by solving the ALS problem \cite{kolda2009tensor} which iteratively computes the SVD of every mode of $\Y$ until the fit $\|\ve(\Y)-(U_1 \otimes U_2 \otimes U_3)\ve(\eS)\|^2_2$ stops to improve. This costs $\mathcal{O}(nr^2)$ per iteration for rank $r$.

\textbf{From MLSVD to GMLSVD:} In our case the fit is given in terms of the pre-computed Laplacian eigenvectors $P_{\mu k}$, i.e, $\|\ve(\Y)-(P_{1k} \otimes P_{2k} \otimes P_{3k})\ve(\X)\|^2_2$ and $\X$ is determined by GCTP eq. \eqref{eq:gctp}. This raises the question about how the factors $U_\mu$ relate to $P_{\mu k}$ and core tensor $\eS$ to $\X$. We argue as following: let $\X = (A_{1k} \kr A_{2k} \kr A_{3k})\ve(\eR)$ be the MLSVD of $\X$. Then, if we set set $V_\mu = P_{\mu k}A_{\mu k}$, then $U_\mu \approx V_\mu$. While we give an example below, a more thoretical study is presented in the Theorem \ref{thm:recovery}.

\textbf{Algorithm for GMLSVD:} Thus, for a tensor $\Y$, one can compute GMLSVD in the following steps: 1) Compute the graph core tensor $\X$ via GCTP (eq. \eqref{eq:gctp}), 2) Perform the MLSVD of $\ve(\X) = (A_{1k} \otimes A_{2k} \otimes A_{3k})\ve(\eR)$, 3) Let the factors $V_\mu = P_{\mu k}A_{\mu k}$ and the core tensor is $\eR$. Given the Laplacian eigenvectors $P_{\mu k}$, GMLSVD scales with $\mathcal{O}(k^4)$ per iteration which is the same as the complexity of solving GCTP.

\textbf{Example:} To understand this, imagine the case of 2D tensor (a matrix) of vectorized wedge images from COIL20 dataset in the columns. $U_1$ for this matrix corresponds to the left singular vectors obtained by SVD and $P_{1k}$ correspond to the first $k$ eigenvectors of the Laplacian $L_1$ between the rows (pixels). Fig. \ref{fig:example_gmlsvd} shows an example wedge image, 1st singular vector in $U_1$ obtained via SVD and the 2nd Laplacian eigenvector $P_{12}$. Clearly, the 1st singular vector in $U_1$ is not equal to the 2nd eigenvector $P_{12}$ (and others). However, if we recover $X$ using GCTP (eq. \eqref{eq:gctp}) and then perform the SVD of $X = A_{1k} R A^\top_{2k}$ and let $V_1 = P_{1k} A_{1k}$, then $U_1 \approx V_1$ (bottom right plot in Fig. \ref{fig:example_gmlsvd}). For more examples, please see Fig. \ref{fig:eg_algo_gmlsvd} in the Appendices.

 \begin{figure}[htbp]
    \centering
        \centering
        \includegraphics[width=0.4\textwidth]{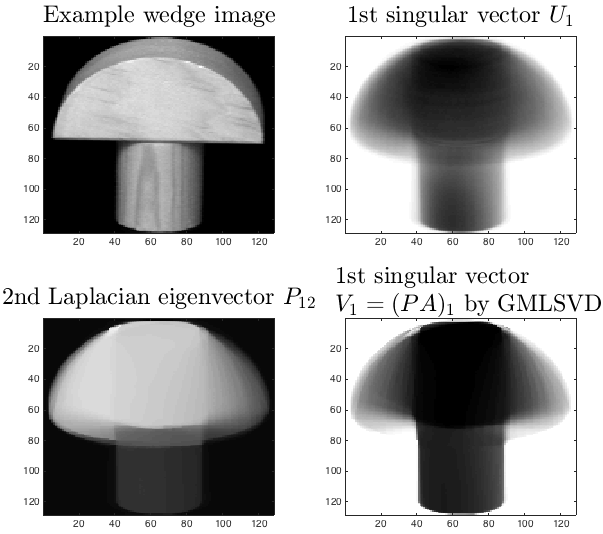}
         \caption{An example wedge image from COIL20 dataset, 1st singular vector in $U_1$ obtained via SVD and the 2nd Laplacian eigenvector $P_{12}$. Clearly, $U_1 \neq P{12}$ (and other eigenvectors). However, $U_1 \approx V_1$, where $V_1 = P_{1k}A_{1k}$, and $A_{1k}$ are the left singular vectors of $X = A_{1k} R A_{2k}^\top$ via GCTP eq. \eqref{eq:gctp}}
        \label{fig:example_gmlsvd}
        
    \end{figure}

% \textbf{Properties of Graph Core Tensor:} As the graph core tensor $\eR$ obtained by GMLSVD is approximately equal to that obtained by MLSVD, it satisfies all the properties satisfied by the later \cite{de2000multilinear}.
% \begin{enumerate}
% \item Two subtensors (fibres) ($\eR_\alpha, \eR_\beta$) from the same mode are orthogonal to each other for $\alpha \neq \beta$.
% \item ordering: $\|\eR_1\|_F \geq  \|\eR_2\|_F \cdots \geq \|\eR_n\|_F$.
% \item The frobenius norms symbolized by $\|\eR_\alpha\|_F = \sigma_\alpha$ are the singular values of that particular mode and $V_\alpha$ are the singular vectors.
% \end{enumerate}

\vspace{-0.4cm}
\subsection{Tensor Robust PCA on Graphs (TRPCAG)}\label{sec:trpcag}
\vspace{-0.2cm}
Another important application of low-rank tensors is Tensor Robust PCA (TRPCA) \cite{tomioka2010estimation}. Unfortunately, this method scales as $\mathcal{O}(n^4)$. We propose an alternate framework, tensor robust PCA on graphs (TRPCAG):
\begin{equation}\label{eq:trpca}
\min_{\X} \|P_{1k}X_1 P^\top_{2,3 k}-Y_1\|_1 + \gamma \sum_\mu \|X_\mu\|_{*g(\Lambda_{\mu k})}.
\vspace{-1.2em}
\end{equation}
The above algorithm requires nuclear norm on $\X \in \Rbb^{k \times k \times k}$ and scales with $\mathcal{O}(nk^2 + k^4)$. This is a significant complexity reduction over TRPCA. We use Parallel Proximal Splitting Algorithm to solve eq. \eqref{eq:trpca} as shown in Appendix \ref{sec:algo_trpca}.

% \subsection{Tensor Clustering}
% \subsection{Tensor Completion}
% \begin{equation}\label{eq:tcomp}
% \min_{\X} \|M_1 \circ (P_{1k}X_1 P^\top_{2,3 k}-Y_1)\|^2_F + \gamma \sum_\mu \|X_\mu\|_{*g(\Lambda_{\mu k})},
% \end{equation}
\vspace{-0.2cm}
\section{Theoretical Analysis}\label{sec:theory}
\vspace{-0.2cm}
Although the inverse problems of the form (eq. \ref{eq:ginvt}) are orders of magnitude faster than the standard tensor based inverse problems, they introduce some approximation. First, note that we do not present any procedure to determine the optimal $k$. Furthermore, as noted from the proof of Theorem \ref{thm:props}, for $L_\mu = P_\mu \Lambda_\mu P_\mu^\top$, the choice of $k$ depends on the eigen gap assumption ($\lambda_{\mu k} \ll \lambda_{\mu k+1}$), which might not exist for the $\K$-Laplacians. Finally, noise in the data adds to the approximation as well. 

We perform our analysis for 2D tensors, i.e, matrices of the form $Y \in \Rbb^{n_1 \times n_2}$. The results can be extended for high order tensors in a straight-forward manner. We assume further that 1) the eigen gaps exist, i.e, there exists a $k^*$, such that $\lambda_{\mu k^*} \ll \lambda_{\mu k^*+1}$ and 2) we select a $k > k^*$ for our method.  For the case of 2D tensor, the general inverse problem \ref{eq:ginvt}, using $\|X\|_{*g(\Lambda_{k})} = \|g(\Lambda_{1k})Xg(\Lambda_{2k})\|_*$ can be written as:
\begin{equation}\label{eq:ginv}
\min_{X} \phi(M(P_{1k}X P^\top_{2k})-Y) + \gamma \|g(\Lambda_{1k})Xg(\Lambda_{2k})\|_*
\end{equation}

\vspace{-0.3cm}
\begin{thm}\label{thm:recovery} 
For any $Y^* \in \MLT$,
\begin{enumerate}
\item Let $Y^* = U_1 S U^\top_2$ be the SVD of $Y$ and $X^* = A_{1k^*} R A^{\top}_{2k^*}$ be the SVD of the GCT $X$ obtained via GCTP (eq.\eqref{eq:gctp}). Now, let $V_\mu = P_{\mu k^*} A_{\mu k^*}, \forall \mu$, where $P_{\mu k^*}$ are the Laplacian eigenvectors of $Y_\mu$, then, $V_\mu = U_\mu$ upto a sign permutation and $S = R$.
\item Solving eq. \eqref{eq:ginv}, with a $k > k^*$ is equivalent to solving the following factorized graph regularized problem:
\begin{align}\label{eq:greg}
\min_{V_1, V_2} & \phi(M(V_1 V^\top_2)-Y) + \gamma_1 \tr(V^\top_1 g(\tilde{L}_1) V_1) \nonumber \\ & + \gamma_2 \tr(V^\top_2 g(\tilde{L}_2) V_2),
\end{align}
where $V_\mu = P_{\mu k}A_{\mu k}$, $X = A_{1k}A^\top_{2k}$ and $g(\tilde{L}_\mu) = P_{\mu k} g(\Lambda_{\mu k}) P^\top_{\mu k}$.
\item Any solution ${F}^* = {V}^*_1 {V}^{*\top}_2\in \Re^{n_1 \times n_2}$ of \eqref{eq:greg}, where $V^*_\mu \in \Rbb^{n \times k}$ and $k > k^*$ with $\gamma_\mu = \gamma/\lambda_{\mu k^*+1}$ and $Y =Z^*_{1}Z^{*\top}_{2} + E$, where $E \in \Rbb^{n_1 \times n_2}$ and $\gamma > 0$ satisfies
\begin{align}\label{eq:theory}
& \phi({F}^* - Y) + \gamma \| \bar{P}^\top_{1k^*}{V}^{*}_{1} \|_F^2 + \gamma \| \bar{P}^\top_{2k^*}{V}^*_{2}\|_F^2 
 \leq \phi(E) \nonumber \\ & + \gamma \Big(\|Z^*_{1}\|_F^2  \frac{g(\lambda_{1 k^*})}{g(\lambda_{1 k^*+1})} + \|Z^*_{2}\|_F^2  \frac{g(\lambda_{2 k^*})}{g(\lambda_{2 k^*+1})} \Big).
\end{align}
where $\lambda_{\mu k^*}, \lambda_{\mu k^*+1} $ denote the $k^{*th}, {k^* + 1}^{st} $ eigenvalues of ${L_\mu}$ and $\bar{P}^\top_{\mu k^*}{V}^{*}_{\mu}$, where $\bar{P}_{\mu k^*} \in \Rbb^{n \times (k - k^*)}$  denote the projection of the factors on the $(k - k^*)$ complement eigenvectors.
\end{enumerate} 
\end{thm}
\vspace{-0.2cm}
\begin{proof}
Please see Appendix \ref{sec:proof_thm_recovery}.
\end{proof}
\vspace{-0.2cm}

\textbf{In simple words:} Theorem \ref{thm:recovery} states that 1) the singular vectors and values of a matrix / tensor obtained by GMLSVD are equivalent to those obtained by MLSVD, 2) in general, the inverse problem \ref{eq:ginv} is equivalent to solving a graph regularized matrix / tensor factorization problem (eq.\ref{eq:greg}) where the factors $V_\mu$ belong to the span of the graph eigenvectors constructed from the modes of the tensor.  The  bound eq. \ref{eq:theory} shows that to recover an MLRTG one should have large eigen gaps ${\lambda_{\mu k^*}}/{\lambda_{ \mu k^*+1}}$. This occurs when the rows of the matricized $\Y$ can be clustered into $k^*$ clusters.  The smaller $\| \bar{P}^\top_{1k^*}{V}^{*}_{1} \|_F^2 + \| \bar{P}^\top_{2k^*}{V}^*_{2}\|_F^2$ is, the closer $F^*$ is to $\MLT$. In case one selects a $k > k^*$, the error is characterized by the projection of singular vectors $V^*_\mu$ on  $(k-k^*)$ complement graph eigenvectors $\bar{P}_{\mu k^*}$. Our experiments show that selecting a $k^* > k$ always leads to a better recovery when the exact value of $k^*$ is not known. 

\vspace{-0.5cm}
\section{Experimental Results}\label{sec:results}
\vspace{-0.3cm}

\begin{figure*}[htbp]
    \centering
        \centering
        \includegraphics[width=1.0\textwidth]{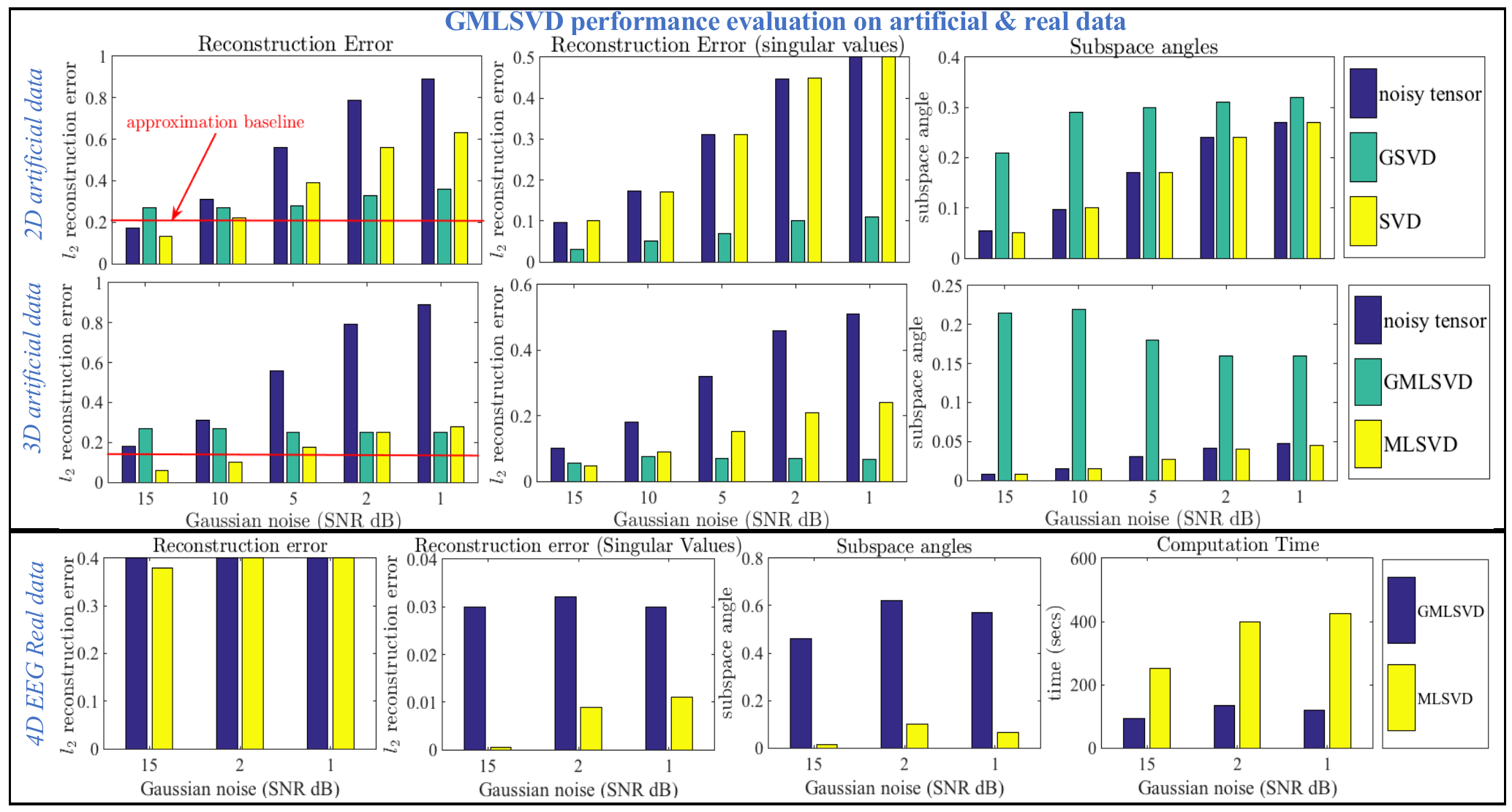}
         \caption{Performance comparison of GSVD with SVD and GMLSVD with MLSVD for 2D artificial (1st row), 3D artificial (2nd row) and 4G real EEG tensors (3rd row) under different SNR scenarios. The artificial tensors have a size 100 and rank 10 along each mode and the 4D EEG tensor has the size $513 \times 128 \times 30 \times 200$. A core tensor of size 30 along each mode and $100 \times 50 \times 30 \times 50$ is used for the artificial and real tensors respectively. The leftmost plots show the $\ell_2$ reconstruction errors, the middle plots show the reconstruction error for top 30 singular values and the right plots show the subspace angles for the 1st five subspace vectors. Clearly GSVD / GMLSVD outperform SVD / MLSVD in terms of computation time and error for big and noisy datasets.}
        \label{fig:gmlsvd_results}
        \vspace{-1.2em}
    \end{figure*}

\textbf{Datasets:} To study the performance of GMLSVD and TRPCAG, we perform extensive experimentation on 4 artificial and 12 real 2D-4D tensors. All types of low-rank artificial datasets are generated by filtering a randomly generated tensor with the $\K$ combinatorial Laplacians constructed from its flattened modes (details in Appendix \ref{sec:experiment_details}). Then, different levels of Gaussian and sparse noise are added to the tensor. The real datasets include Hyperspectral images, EEG, BCI, FMRI, 2D and 3D image and video tensors and 3 point cloud datasets.

\textbf{Methods:}  GMLSVD and TRPCAG are low-rank tensor factorization methods, which are programmed using GSPBox \cite{perraudin2014gspbox}, UNLocBox \cite{perraudin2014unlocbox} and Tensorlab \cite{tensorlab} toolboxes. Note that GMLSVD is robust to Gaussian and TRPCAG to sparse noise, therefore, these methods are tested under varying levels of these two types of  noise.  To avoid any confusion, we call the 2D tensor (matrix) version of GMLSVD as Graph SVD (GSVD). 

For the 3D tensors with Gaussian noise we compare GMLSVD performance with MLSVD. For the 3D tensor with sparse noise we compare TRPCAG  with Tensor Robust PCA (TRPCA) \cite{tomioka2010estimation}  and GMLSVD. For the 2D tensors (matrices) with Gaussian noise we compare GSVD with simple SVD.  Finally for the 2D matrix with sparse noise we compare TRPCAG with Robust PCA (RPCA) \cite{candes2011robust}, Robust PCA on Graphs (RPCAG) \cite{shahid2015robust},  Fast Robust PCA on Graphs (FRPCAG) \cite{shahid2015fast} and Compressive PCA (CPCA) \cite{shahid2016compressive}. Not all the methods are tested on all the datasets due to computational reasons.

\textbf{Parameters:} For all the experiments involving TRPCAG and GMLSVD, the $\K$ graphs are constructed from the rows of each of the flattened modes of the tensor, using $\K = 10$ and a Gaussian kernel for weighting the edges. For all the other methods the graphs are constructed as required, using the same parameters as above. Each method has several hyper-parameters which require tuning. For a fair comparison, all the methods are properly tuned for their hyper-parameters and best results are reported. For details on all the datasets, methods and parameter tuning please refer to Appendix \ref{sec:experiment_details}.

\textbf{Evaluation Metrics:} The metrics used for the evaluation can be divided into two types: 1) quantitative and 2) qualitative. Three different types of quantitative measures are used: 1) normalized $\ell_2$ reconstruction error of the tensor $\|\ve(\Y)-\ve(\Y^*)\|_2/\|\ve(\Y^*)\|_2$, 2) the normalized $\ell_2$ reconstruction error of the first $k^*$ (normally $k^* = 30$) singular values along mode 1 $\|\sigma^1_{1:k^*}-\sigma^{1*}_{1:k^*}\|_2/\|\sigma^{1*}_{1:k^*}\|_2$  3) the subspace angle (in radian) of mode 1 between the 1st five subspace vectors determined by the proposed method and those of the clean tensor: $\arccos{|V_1(:,1:5)^\top U_1(:,1:5)|}$ and 4) the alignment of the singular vectors $\diag(|V_1(:,1:5)^\top U_1(:,1:5)|)$, where $V_1$ and $U_1$ denote the mode 1 singular vectors determined by the proposed method and clean tensor. The qualitative measure involves the visual quality of the low-rank components of tensors.

\subsection{Experiments on GMLSVD}
\vspace{-0.2cm}
\textbf{Performance study on artificial datasets:} The first two rows of Fig. \ref{fig:gmlsvd_results} show the performance of GSVD (for 2D) and GMLSVD (for 3D) on artificial tensors of the size 100 and rank 10 in each mode, for varying levels of Gaussian noise ranging from 15dB to 1dB. The three plots show the $\ell_2$ reconstruction error of the recovered tensor, the first $k^* = 30$ singular values and and the subspace angle of the 1st mode subspace (top 5 vectors), w.r.t to those of the clean tensor. These results are compared with the standard SVD for 2D tensor and standard MLSVD for the 3D tensor. It is interesting to note from the leftmost plot that the $\ell_2$ reconstruction error for GSVD tends to get lower as compared to SVD at higher noise levels (SNR less than 5dB). The middle plot explains this phenomena where one can see that the $\ell_2$ error for singular values is significantly lower for GSVD than SVD at higher noise levels. This observation is logical, as for higher levels of noise the lower singular values are also affected. SVD is a simple singular value thresholding method which does not eliminate the effect of noise on lower singular values, whereas GSVD is a smart weighted nuclear norm method which thresholds the lower singular values via a function of the graph eigenvalues. This effect is shown in detail in Fig. \ref{fig:gsvd_example_small}. On the contrary, the subspace angle (for first 5 vectors) for GSVD is higher than SVD for all the levels of noise. This means that the subspaces of the GSVD are not well aligned with the ones of the clean tensor. However, as shown in the right plot in Fig. \ref{fig:gsvd_example_small}, strong first 7 elements of the diagonal $\diag|V_1(:,1:5)^\top U_1(:,1:5)|$ show that the individual subspace vectors of GSVD  are very well aligned with those of the clean tensor.  This is the primary reason why the $\ell_2$ reconstruction error of GSVD is less as compared to SVD at low SNR. At higher SNR, the approximation error of GSVD dominates the error due to noise. Thus, GSVD reveals its true power at low SNR scenarios. Similar observations can be made about GMLSVD from the 2nd row of Fig. \ref{fig:gmlsvd_results}.
 
 \begin{figure}[htbp]
    \centering
        \centering
        \includegraphics[width=0.45\textwidth]{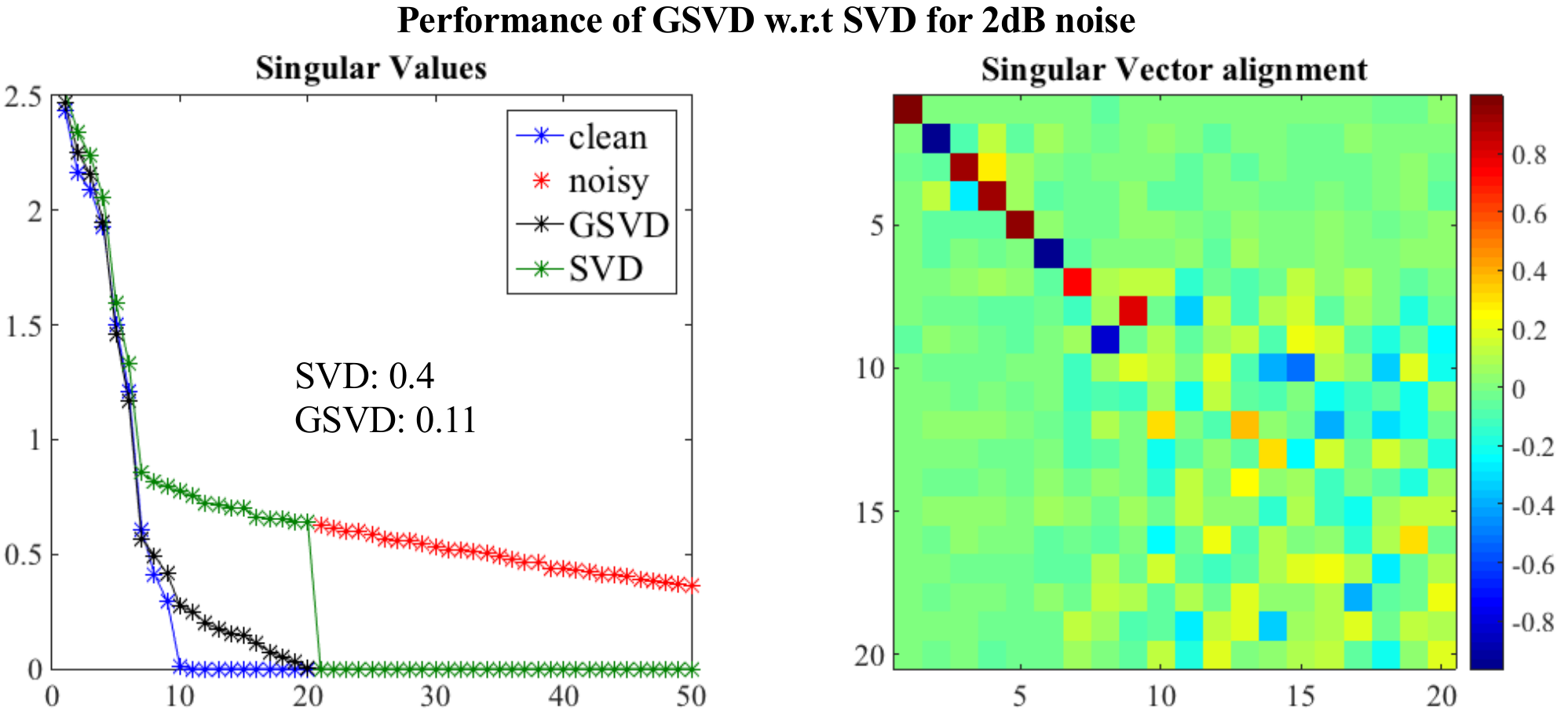}
         \caption{Singular values (left plot) of GSVD, SVD and clean data and singular vector alignment of GSVD and clean data (right plot) for a 2D matrix corrupted with Gaussian noise (SNR 5dB). Clearly GSVD eliminates the effect of noise from lower singular values and aligns the first few singular vectors appropriately. }
        \label{fig:gsvd_example_small}
        \vspace{-1.2em}
    \end{figure}

\textbf{Time, memory \& performance on real datasets:} The 3rd row of Fig. \ref{fig:gmlsvd_results} shows the results of GMLSVD compared to MLSVD for a 4D real EEG dataset of size 3GB and dimensions $513 \times 128 \times 30 \times 200$. A core tensor of size $100 \times 50 \times 30 \times 50$ is used for this experiment. It is interesting to note that for low SNR scenarios the $\ell_2$ reconstruction error of both methods is approximately same. GMLSVD and MLSVD both show a significantly lower error for the singular values (note that the scale is small), whereas GMLSVD's subspaces are less aligned as compared to those of MLSVD. The rightmost plot of this figure compares the computation time of both methods. \textit{Clearly, GMLSVD wins on the computation time (120 secs) significantly as compared to MLSVD (400 secs)}. For this 3GB dataset, GMLSVD requires 6GB of memory whereas MLSVD requires 15GB, as shown in the detailed results in Fig. \ref{fig:real_data_gaussian_main} of the Appendices. Fig. \ref{fig:real_data_gaussian_main} also presents results for BCI, hyperspectral and FMRI tensors and reveals how GMLSVD performs better as compared to MLSVD while requiring less computation time and memory for big datasets in the low SNR regime. For a visualization of the clean, noisy and GMLSVD recovered tensors, their singular values and the alignments of the subspace vectors for FMRI and hyperspectral tensors, please refer to Fig. \ref{fig:gmlsvd_main} in Appendices.

   \begin{figure}[htbp]
    \centering
        \centering
        \includegraphics[width=0.45\textwidth]{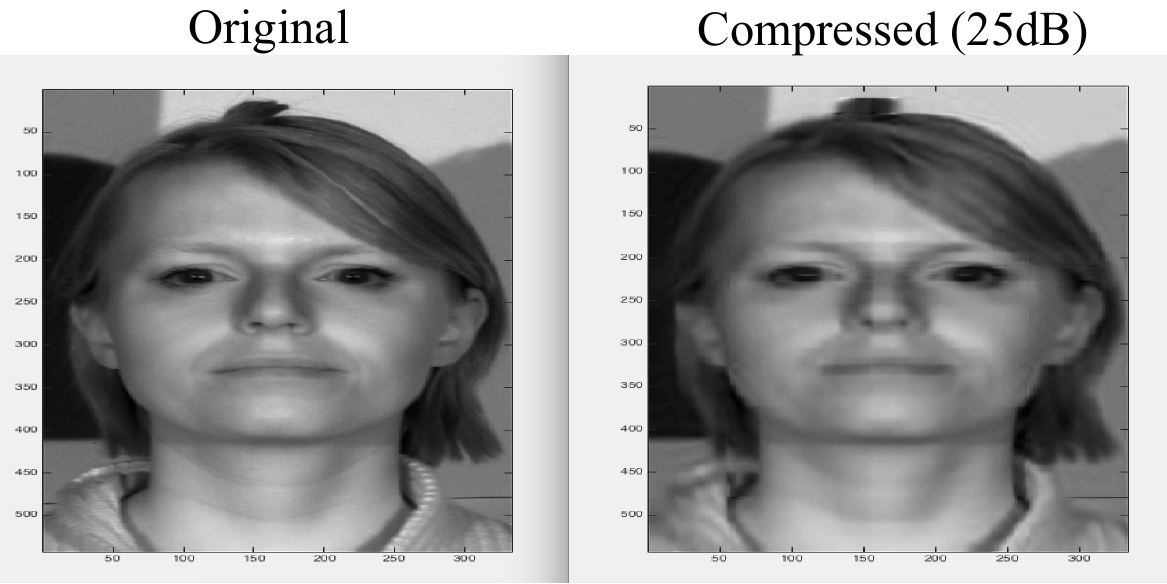}
         \caption{Qualitative and quantitative results for the face ($542 \times 333 \times 148$) 3D hyperspectral tensors. Using core of size $70 \times 70 \times 30$  we attained 150 times compression while maintaining an SNR of 25dB.}
        \label{fig:hyperspectral_compression}
        \vspace{-1.2em}
    \end{figure}
    
     \begin{figure*}[htbp]
    \centering
        \centering
        \includegraphics[width=1.0\textwidth]{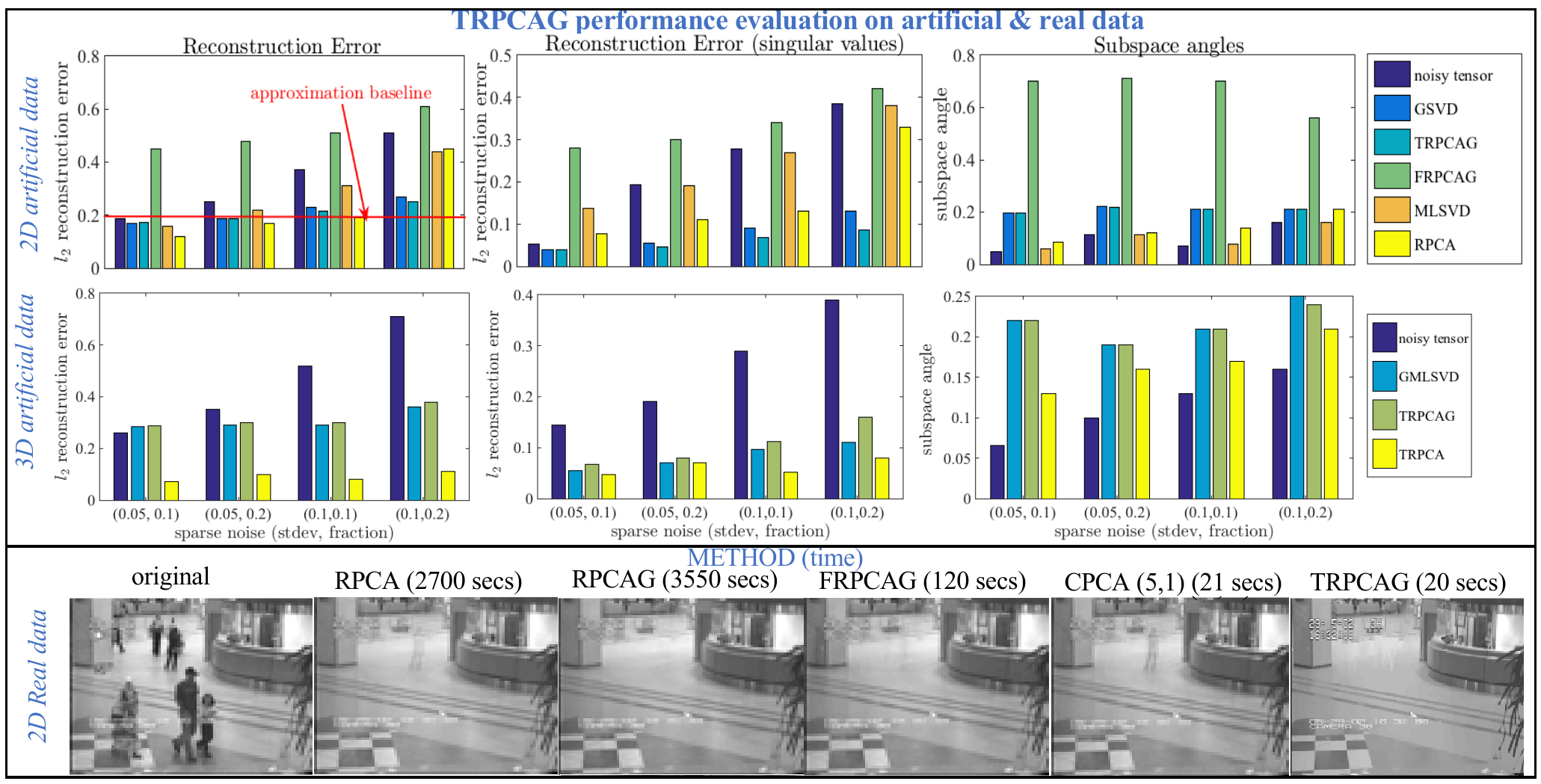}
         \caption{Performance analysis of TRPCAG for 2D artificial (1st row) and 3D artificial (2nd row) with varying levels of sparse noise and 2D airport lobby video tensor (3rd row). The artificial tensors have a size 100 and rank 10 along each mode. A core tensor of size 30 along each mode is used for the artificial tensors and $100 \times 50$ for the real tensor. The leftmost plots show the $\ell_2$ reconstruction errors, the middle plots show the $\ell_2$ reconstruction error for 30 singular values and the right plots show the subspace angles for the 1st five subspace vectors. For the 2D real video dataset, the rightmost frame shows the result obtained via TRPCAG. The computation time for different methods is written on the top of frames. Clearly TRPCAG performs quite well while significantly reducing the computation time.}
       \label{fig:trpcag_results}
       \vspace{-1.2em}
    \end{figure*}
    
\begin{figure}[htbp]
    \centering
        \centering
        \includegraphics[width=0.45\textwidth, height = 0.15\textwidth]{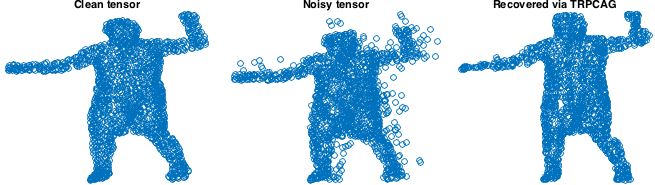}
         \caption{TRPCAG performance for recovering the low-rank point cloud of a dancing person from the sparse noise. Actual point cloud (left), noisy point cloud (middle), recovered via TRPCAG (right).}
        \label{fig:dancer}
        \vspace{-1.2em}
    \end{figure}
    
    \begin{figure}[htbp]
    \centering
        \centering
        \includegraphics[width=0.45\textwidth]{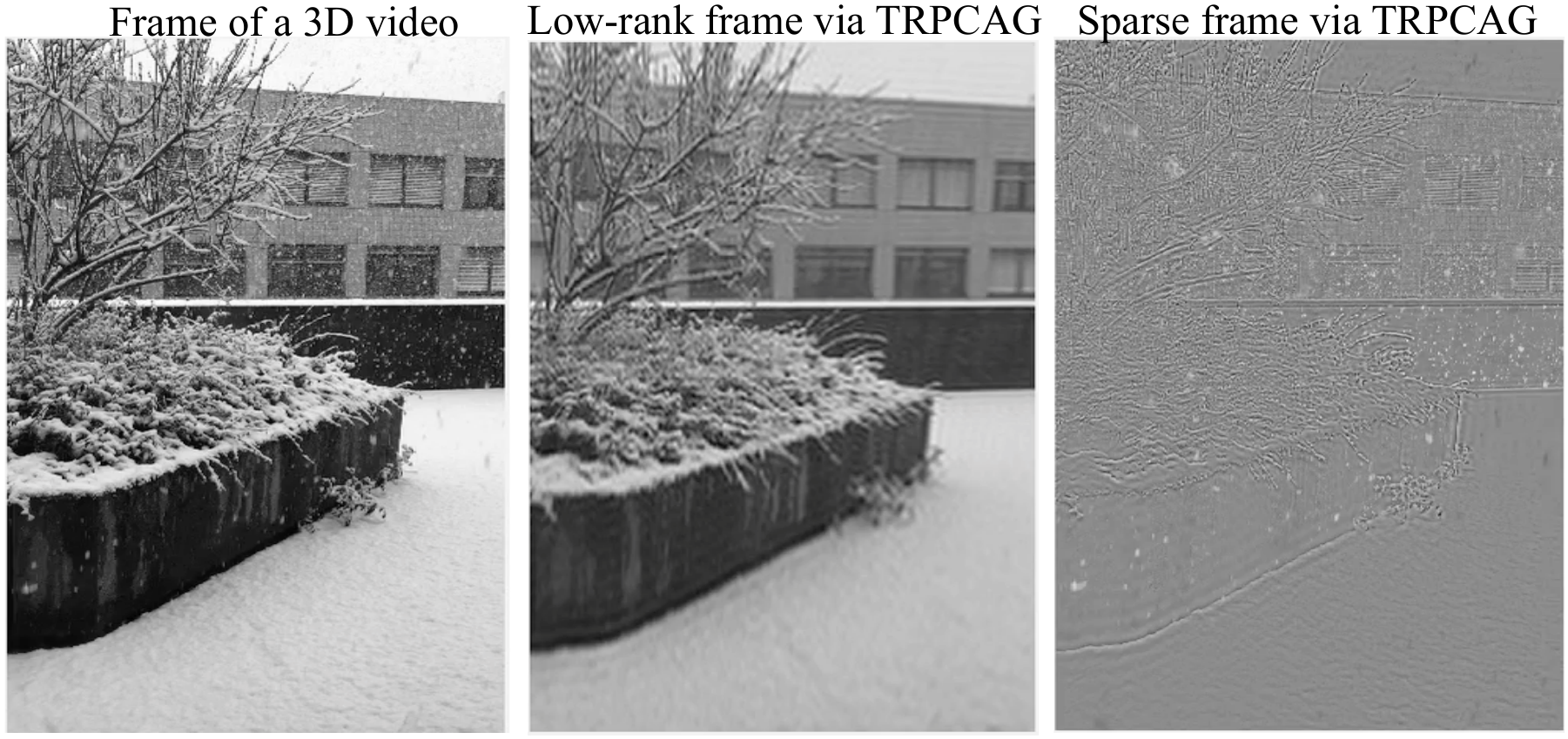}
         \caption{Low-rank recovery for a 3D video of dimension $1920 \times 1080 \times 500$ and size 1.5GB via TRPCAG. Using a core size of $100 \times 100 \times 50$, TRPCAG converged in less than 3 minutes. }
        \label{fig:snow}
        \vspace{-1.2em}
    \end{figure}
    
     \begin{figure*}[htbp]
    \centering
        \centering
        \includegraphics[width=1.0\textwidth]{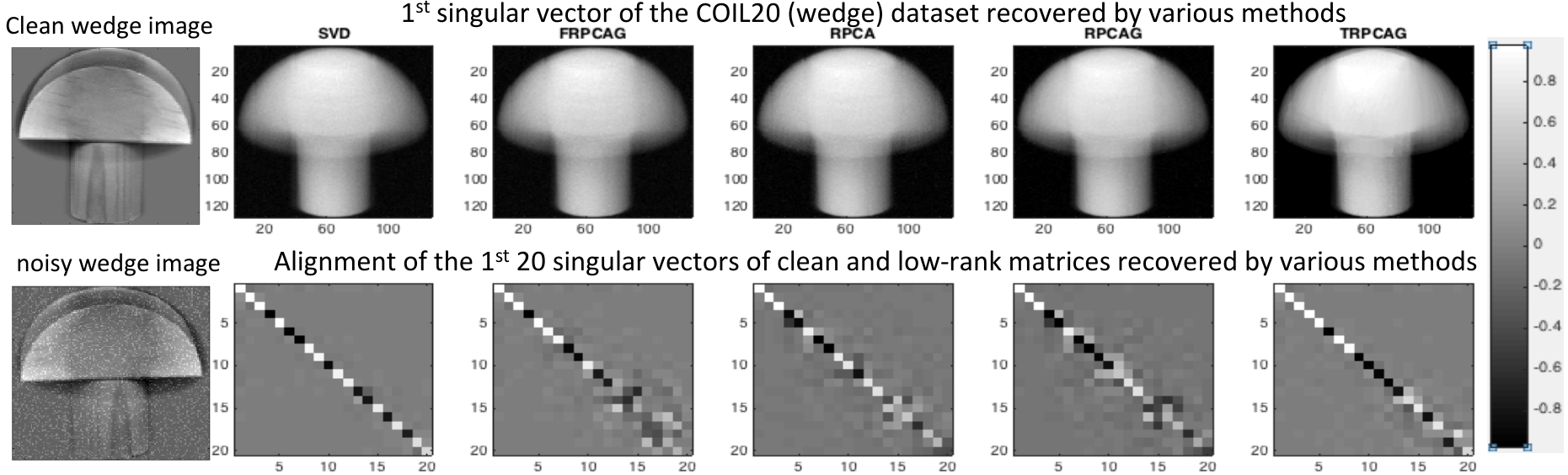}
         \caption{Robust recovery of subspace structures via TRPCAG. The leftmost plots show a clean and sparsely corrupted sample wedge image from COIL20 dataset. Other plots in the 1st row show the 1st singular vector recovered by various low-rank recovery methods, SVD, FRPCAG, RPCA, RPCAG and TRPCAG and the 2nd row shows the alignment of the 1st 20 singular vecors recovered by these methods with those of the clean tensor. }
        \label{fig:sv_wedge}
        \vspace{-1.2em}
    \end{figure*}

\textbf{Compression:} An obvious goal of MLSVD is the compression of low-rank tensors, therefore, GMLSVD can also be used for this purpose. Fig. \ref{fig:hyperspectral_compression} shows results for the face ($542 \times 333 \times 148$) 3D hyperspectral tensor. Using core of size $70 \times 70 \times 30$  we attained 150 times compression while maintaining SNR of 25dB. Fig. \ref{fig:hyperspectral_compression_zoomed} in the Appendices shows such results for three other datasets. The rightmost plots of Fig. \ref{fig:gmlsvd_main} in Appendices also show compression results for FMRI, EEG and BCI datasets.

\vspace{-0.2cm}
\subsection{Experiments on TRPCAG}
\vspace{-0.2cm}
\textbf{Performance study on artificial datasets:} The first two rows of Fig. \ref{fig:trpcag_results} show experiments on the 2D and 3D artificial datasets with varying levels of sparse noise. The 2D version of TRPCAG is compared with state-of-the-art Robust PCA based methods, FRPCAG and RPCA and also with SVD based methods like MLSVD and GSVD. Conclusions, similar to those for the Gaussian noise experiments can be drawn for the 2D matrices (1st row). Thus TRPCAG is better than state-of-the-art methods in the presence of large fraction of sparse noise. For the 3D tensor (2nd row), TRPCAG is compared with  GMLSVD and TRPCA \cite{tomioka2010estimation}. Interestingly, the performance of TRPCA is always better for this case, even in the presence of high levels of noise. While TRPCA produces the best results, its computation time is orders of magnitude more than TRPCAG (discussed next). A detailed analysis of the singular values recovered by TRPCAG can be done via Fig. \ref{fig:TRPCAG_example} in the Appendices.

\textbf{Time \& performance on 2D real datasets:} The 3rd row of Fig. \ref{fig:trpcag_results} present experiments on the 2D real video dataset obtained from an airport lobby (every frame vectorized and stacked as the columns of a matrix). The goal is to separate the static low-rank component from the sparse part (moving people) in the video. The results of TRPCAG are compared with RPCA, RPCAG, FRPCAG and CPCA with a downsampling factor of 5 along the frames. Clearly, TRPCAG recovers a low-rank which is qualitatively equivalent to the other methods in a time which is 100 times less than RPCA and RPCAG and an order of magnitude less as compared to FRPCAG. Furthermore, TRPCAG requires the same time as  sampling based CPCA 
method but recovers a better quality low-rank structure as seen from the 3rd row. The performance quality of TRPCAG is also evident from the point cloud experiment in Fig. \ref{fig:dancer} where we recover the low-rank point cloud of a dancing person after adding sparse noise to it. Experiments on two more videos (shopping mall and escalator) and two point cloud datasets (walking dog and dancing girl) are presented in Figs. \ref{fig:video_trpcag_extra}, \ref{fig:escalator} \& \ref{fig:point_cloud} in Appendices.

\textbf{Scalability of TRPCAG on 3D video:} To show the scalability of TRPCAG as compared to TRPCA, we made a video of snowfall at the campus and tried to separate the snow-fall from the low-rank background via both methods. For this 1.5GB video of dimension $1920 \times 1080 \times 500$, TRPCAG (with core tensor size $100 \times 100 \times 50$) took less than 3 minutes, whereas TRPCA did not converge even in 4 hours. The result obtained via TRPCAG is visualized in Fig. \ref{fig:snow}. The complete videos of actual frames, low-rank and sparse components, for all the above experiments are provided with the supplementary material of the paper.

\vspace{-0.3cm} 
\subsection{Robust Subspace Recovery}
\vspace{-0.2cm} 
It is imperative to point out that the inverse problems of the form eq. \ref{eq:ginvt} implicitly determine the subspace structures (singular vectors) from grossly corrupted tensors. Examples of GMLSVD from Section \ref{sec:gmlsvd} correspond to the case of clean tensor. In this section we show that the singular vectors recovered by TRPCAG (Section \ref{sec:trpcag}) from the sparse and grossly corrupted tensors also align closely with those of the clean tensor. Fig. \ref{fig:sv_wedge} shows the example of the same wedge image from 2D COIL20 dataset that was used in Section \ref{sec:gmlsvd}. The leftmost plots show a clean and sparsely corrupted sample wedge image. Other plots in the 1st row show the 1st singular vector recovered by various low-rank recovery methods, SVD, FRPCAG, RPCA, RPCAG and TRPCAG and the 2nd row shows the alignment of the 1st 20 singular vecors recovered by these methods with those of the clean tensor. The rightmost plots correspond to the case of TRPCAG and clearly show that the recovered subspace is robust to sparse noise. Examples of YALE, COIL20 and ORL datasets are also shown in Fig. \ref{fig:singular_vectors} in the Appendices.

\vspace{-0.2cm}
\subsection{Effect of parameters} 
\vspace{-0.2cm}
Due to space constraints we study the effect of parameters $\gamma$, the multilinear rank $k$ and the power $\alpha$ of $g(\cdot)$ in eq. \ref{eq:ginvt} in Fig. \ref{fig:params} of the Appendices. To summarize, once the parameters $\gamma$ and $k$ are tuned, eq. \ref{eq:ginvt} becomes  insensitive to parameter $\alpha$.

\vspace{-0.2cm}
\section{Conclusion}
\vspace{-0.2cm}
Inspired by the fact that the first few eigenvectors of the $\K$-graph provide a smooth basis for data, we present a graph based low-rank tensor decomposition model. Any low-rank tensor can be decomposed as a multilinear combination of the lowest $k$ eigenvectors of the graphs constructed from the rows of the flattened modes of the tensor (MLRTG). We propose a general tensor based convex optimization framework which overcomes  the computational and memory burden of standard tensor problems and enhances the performance in the low SNR regime. More specifically we demonstrate two applications of MLRTG 1) Graph based MLSVD and 2) Tensor Robust PCA on Graphs for 4 artificial and 12 real datasets under different noise levels. Theoretically, we prove the link of MLRTG to the joint stationarity of signals on graphs. We also study the performance guarantee of the proposed general optimization framework by connecting it to a factorized graph regularized problem.

\clearpage
{\small
\bibliographystyle{ieee}
\bibliography{pcabib}
}

\onecolumn
\clearpage
\newpage
\appendix

\section{Appendices}

\subsection{Proof of Theorem \ref{thm:props}}\label{sec:proof_props}
In the main body of the paper, we assume (for simplicity of notation) that the graph multilinear rank of the tensor is equal in all the modes $(k,k,k)$. However, for the proof of this Theorem we adopt a more general notation and assume a different rank $k_\mu$ for every mode $\mu$ of the tensor.

We first clearly state the properties of MLRTG:

\begin{enumerate}
\item{ \textbf{Property 1: Joint Approximate Graph Stationarity:} \begin{defn}\label{def:jgwss}
A tensor $\Y^* \in \MLT$ satisfies Joint Approximate Graph Stationarity, i.e, 
its ${\mu}^{th}$ matricization / flattening  $Y_\mu $ satisfies approximate graph stationarity  $\forall \mu$:
\begin{equation}\label{eq:gsc}
s(\Gamma_\mu) = \frac{\|\di(\Gamma_\mu)\|^2_F}{\|\Gamma_\mu\|^2_F} \approx 1, 
\end{equation}
\end{defn}}

\item {\textbf{Property 2: Low Frequency Energy Concentration:} \begin{defn}\label{def:pc}
  A tensor $\Y^* \in \MLT$ satisfies the low-frequency energy concentration property, i.e, the energy is concentrated in the top  entries of the graph spectral covariance matrices $\Gamma_\mu, \forall \mu = 1,2,\cdots,d$. 
  \begin{align}\label{eq:psd}
& \hat{s}(\Gamma_\mu)  = \frac{\|\Gamma_\mu(1:k_\mu, 1:k_\mu)\|^2_F}{\|\Gamma_\mu\|^2_F}  \approx 1 
\end{align}
  \end{defn}
}
\end{enumerate}

\subsubsection{Assumption: The existence of Eigen Gap} For the proof of this Theorem, we assume that the `eigen gap condition' holds. In order to understand this condition, we guide the reader step by step by defining the following terms:
\begin{enumerate}
\item Cartesian product of graphs
\item Eigen gap of a cartesian product graph
\end{enumerate}

\begin{defn}[\textbf{Eigen Gap of a graph}]\label{def:egap}
A graph Laplacian $L$ with an eigenvalue decomposition $L = P\Lambda P^\top$  is said to have an eigen gap if there exists a $k > 0$, such that $\lambda_{k} \ll \lambda_{k+1}$ and $\lambda_k / \lambda_{k+1} \approx 0$.
\end{defn}

\textbf{Separable Eigenvector Decomposition of a graph Laplacian:}
The eigenvector decomposition of a combinatorial Laplacian $L$ of a graph $\G$ possessing an eigen gap (satisfying Definition \ref{def:egap}) can be written as:
$$L = P\Lambda P^\top = P_k \Lambda_k P_k^\top +  \bar{P}_k \bar{\Lambda}_k \bar{P}_k^\top,$$
where  $P_k \in \Rbb^{n \times k}, \bar{P}_k \in \Rbb^{n \times (n - k)}, \Lambda_k \in \Rbb^{k\times k}, \bar{\Lambda}_k \in \Rbb^{(n-k) \times (n-k)}$  denote the first $k$ low frequency eigenvectors and eigenvalues in $P, \Lambda$.

 For a $\K$-nearest neighbors graph constructed from a $k_\mu$-clusterable data $Y_\mu$ (along rows) one can expect $\lambda_{\mu k_\mu}/\lambda_{\mu k_\mu + 1} \approx 0$ as $\lambda_{\mu k_\mu} \approx 0$ and $\lambda_{\mu k_\mu} \ll \lambda_{\mu k_\mu + 1}$.

\begin{defn}[\textbf{Cartesian product}]\label{def:cp}
Suppose we have two graphs $\G_1 (\V_1, \E_1, W_1, \D_1)$ and $\G_2(\V_2, \E_2, W_2, \D_2)$ where the tupple represents (vertices, edges, adjacency matrix, degree matrix). The Cartesian product $\G = \G_1 \times \G_2$ is a graph such that the vertex set  is the Cartesian product $\V=\V_1 \times \V_1$ and the edges are set according to the following rules: any two vertices $(u_1,v_1)$ and $(u_2,v_2)$ are adjacent in $\G$ if and only if either
\begin{itemize}
\item  $u_1 = u_2$ and $v_1$ is adjacent with $v_2$ in $\G_2$
\item  $v_1 = v_2$ and $u_1$ is adjacent with $u_2$ in $\G_1$.
\end{itemize}
    
The adjacency matrix of the Cartesian product graph $\G$ is given by the matrix Cartesian product:
    \[
    W=W_1 \times W_2 = W_1\otimes I_2 + I_1 \otimes W_2
    \]
\end{defn} 
\begin{lemma}\label{lem:cp}
The degree matrix $\D$ of the graph $\G = \G_1 \times \G_2$ which satisfies definition \ref{def:cp} is given as
\[
\D = \D_1 \otimes I_2 + I_1 \otimes \D_2 = \D_1 \times \D_2
\]
The \textbf{combinatorial} Laplacian of $\G$ is:
\[
L = L_1\otimes I_2 + I_1\otimes L_2 = L_1\times L_2
\] 
\end{lemma}
\begin{proof}
The adjacency matrix of the Cartesian product graph $\G$ is given by the matrix Cartesian product:
    \[
    W=W_1\times W_2 = W_1\otimes I_2 + I_1 \otimes W_2
    \]

With this definition of the adjacency matrix, it is possible to write the degree matrix of the Cartesian product graph as cartesian product of the factor degree matrices:
\begin{align*}
d &= W(\ones_1\otimes \ones_2) = (W_1 \otimes I_2 + I_1 \otimes W_2)(\ones_1\otimes \ones_2)\\
&=(W_1 \otimes I_2) (\ones_1\otimes \ones_2)+(I_1 \otimes W_2) (\ones_1\otimes \ones_2)\\
&=(W_1 \ones_1 ) \otimes(I_2 \ones_2) + (I_1\ones_1)\otimes ( W_2 \ones_2)\\
&= d_1 \otimes \ones_1 + \ones_2 \otimes d_2 
\end{align*}
where we have used the following property
\begin{equation*}\label{eq:kron2}
({A_1\otimes B_1})({A_2\otimes B_2}) = ({A_1 A_2})\otimes({B_1 B_2}).
\end{equation*}
This implies the following matrix equality:
\[
\D = \D_1 \otimes I_2 + I_1 \otimes \D_2 = \D_1 \times \D_2
\]
The \textbf{combinatorial} Laplacian of the cartesian product is:
\begin{align*}
L &= \D-W = \D_1 \times \D_2 -W_1\times W_2\\
  &= (\D_1\otimes I_2 + I_1 \otimes \D_2) - (W_1\otimes I_2 + I_1 \otimes W_2)\\
  &= (\D_1-W_1)\otimes I_2 + I_1\otimes (\D_2-W_2) \\
  &= L_1\otimes I_2 + I_1\otimes L_2 = L_1\times L_2
\end{align*}
\end{proof}

 For our purpose we define the Eigen gap of the Cartesian Product Graph as follows:

\begin{defn}[\textbf{Eigen Gap of a Cartesian Product Graph}]\label{def:sgap}
A cartesian product graph as defined in \ref{def:cp},  is said to have an eigen gap if there exist $k_1, k_2$, such that, $ \max\{\lambda_{1 k_1},\lambda_{2 k_2}\} \ll \min\{{\lambda}_{1 k_1 +1},{\lambda}_{2 k_2 +1}\}$, where $\lambda_{\mu k_\mu}$ denotes the $k_\mu$ eigenvalue of the ${\mu}^{th}$ graph Laplacian $L_\mu$.
\end{defn}

 \subsubsection{Consequence of `Eigen gap assumption': Separable eigenvector decomposition of a cartesian product graph}
 
The eigen gap assumption (definition \ref{def:sgap}) is important to define the notion of the `Separable eigenvector decomposition of a cartesian product graph' which will be used in the final steps of this Theorem. We define the eigenvector decomposition of a Laplacian of cartesian product of two graphs. The definition can be extended in a straight-forward manner for more than two graphs as well.

\begin{lemma}[\textbf{Separable Eigenvector Decomposition of a Cartesian Product Graph}]\label{lem:EV}
For a cartesian product graph as defined in \ref{def:cp}, the eigenvector decomposition can be written in a separable form as:
\begin{align*}
L & = (P_1\otimes P_2)(\Lambda_1 \times \Lambda_2)( P_1 \otimes P_2)^\top = P_{k} \Lambda_k P_k^\top + \bar{P}_{k} \bar{\Lambda}_k \bar{P}_k^\top  = P\Lambda P^\top, 
\end{align*} 
where  $P_k \in \Rbb^{n_1 n_2 \times k_1 k_2}, \bar{P}_k \in \Rbb^{n_1 n_2 \times (n_1 - k_1 )(n_2 - k_2) }, \Lambda_k \in \Rbb^{k_1 k_2 \times k_1 k_2}, \bar{\Lambda}_k \in \Rbb^{(n_1-k_1)(n_2-k_2) \times (n_1-k_1)(n_2-k_2)}$ and $k$ denotes the first $k_1 k_2$ low frequency eigenvectors and eigenvalues in $P, \Lambda$.
\end{lemma}

\begin{proof}
The eigenvector matrix of the cartesian product graph can be derived as:
\begin{align*}
L &= L_1\otimes I_2 + I_1\otimes L_2 = (P_1\Lambda_1 P_1^\top)\otimes I_2 +I_1\otimes (P_2\Lambda_2 P_2^\top)\\
&=(P_1\Lambda_1 P_1^\top)\otimes (P_2 I_2 P_2^\top) + (P_1 I_1 P_1^\top)\otimes (P_2\Lambda_2 P_2^\top)\\
&=(P_1\otimes P_2)(\Lambda_1 \otimes I_2)( P_1^\top \otimes P_2^\top) + (P_1\otimes P_2)(I_1 \otimes \Lambda_2)( P_1^\top \otimes P_2^\top)\\
&=(P_1\otimes P_2)(\Lambda_1 \otimes I_2 + I_1 \otimes \Lambda_2)( P_1 \otimes P_2)^\top\\
&= (P_1\otimes P_2)(\Lambda_1 \times \Lambda_2)( P_1 \otimes P_2)^\top = P\Lambda P^\top
\end{align*}
So the eigenvector matrix is given by the Kronecker product between the eigenvector matrices of the factor graphs and the eigenvalues are the element-wise  summation  between all the possible pairs of factors eigenvalues, i.e. the cartesian product between the eigenvalue matrices.

\begin{align}\label{eq:laps}
L_1 \kr I_2 + I_1 \kr L_2 & = (P_1 \kr P_2)(\Lambda_1 \times \Lambda_2)(P_1 \kr P_2)^\top \nonumber \\
& = \big((P_{1 k_1} + \bar{P}_{1 k_1}) \kr (P_{2 k_2} + \bar{P}_{2 k_2})\big)\big((\Lambda_{1 k_1} + \bar{\Lambda}_{1 k_1}) \times (\Lambda_{2 k_2} + \bar{\Lambda}_{2 k_2})\big)\big((P_{1 k_1} + \bar{P}_{1 k_1})\kr (P_{2 k_2} + \bar{P}_{2 k_2})\big)^\top,
\end{align}
where we assume that $P_{k_\mu} \in \Rbb^{n_\mu \times n_\mu}$ with the first $k_\mu$ columns in $P_1$ and 0 appended for others and $\bar{P}_{k_\mu} \in \Rbb^{n_\mu \times n_\mu}$ with the first $k_\mu$ columns equal to 0 and others copied from $P_1$. The same holds for $\Lambda_\mu, \bar{\Lambda}_\mu$ as well. Now
\begin{align*}
(P_{1 k_1} + \bar{P}_{1 k_1}) \kr (P_{2 k_2} + \bar{P}_{2 k_2}) & = \big(P_{1 k_1} \kr P_{2 k_2} + P_{1 k_1}\kr \bar{P}_{2 k_2} + \bar{P}_{1 k_1} \kr  P_{2 k_2} +\bar{P}_{1 k_1} \kr  \bar{P}_{2 k_2} \big) \nonumber \\
& = P_{k_1 k_2} + \bar{P}_{k_1 k_2}, 
\end{align*}
where we use $P_{k_1 k_2} = P_{1 k_1} \kr P_{2 k_2}$. Now removing the zero appended columns we get $P = (P_{k_1 k_2},\bar{P}_{k_1 k_2})$. Now, let
\begin{align*}
(\Lambda_{1 k_1} + \bar{\Lambda}_{1 k_1}) \times (\Lambda_{2 k_2} + \bar{\Lambda}_{2 k_2}) = \Lambda_{k_1 k_2} + \bar{\Lambda}_{k_1 k_2}
\end{align*}
removing the zero appended entries we get $\Lambda = (\Lambda_{k_1 k_2},\bar{\Lambda}_{k_1 k_2})$. For a $\K$-nearest neighbors graph constructed from a $k_1$-clusterable data  (along rows) one can expect $\lambda_{k_1}/{\lambda}_{k_1 + 1} \approx 0$ as $\lambda_{k_1} \approx 0$ and $\lambda_{k_1} \ll {\lambda}_{k_1 + 1}$. Furthermore $\displaystyle\max_{\mu}(\lambda_{ k_{\mu}})\ll \min_{\mu}({\lambda}_{ k_{\mu}+1})$ (definitions \ref{def:egap} \& \ref{def:sgap}). Thus eq. \eqref{eq:laps} can be written as:
\begin{align*}
L_1 \kr I_2 + I_1 \kr L_2 & =  [P_{k_1 k_2} | \bar{P}_{k_1 k_2}] [\Lambda_{k_1 k_2} |\bar{\Lambda}_{k_1 k_2}] [P_{k_1 k_2}| \bar{P}_{k_1 k_2}]^\top
\end{align*}
\end{proof}

\subsubsection{Final steps}

Now we are ready to prove the Theorem. We start by expanding the denominator of the expression from property 2 above. We start with the GSC $\Gamma_\mu$ for any $Y_\mu$ and use $C_\mu = Y_\mu Y^\top_\mu$. For a 3D tensor, we index its modes with $\mu \in \{1,2,3\}$. While considering the eigenvectors of $\mu^{th}$ mode, i.e, $P_{\mu k_\mu}$ we represent the kronecker product of the eigenvectors of other modes as $ P_{-\mu k_{-\mu}}$ for simplicity:
\begin{align}\label{eq:gsc_expand}
\Gamma_\mu  & = P^\top_\mu C_\mu P_\mu \nonumber\\
     & = P^\top_\mu Y_\mu Y^\top_\mu  P_\mu \nonumber\\
     &  = P^\top_\mu [ P_{\mu k_\mu} X_\mu P^\top_{-\mu k_{-\mu}}+ \bar{P}_{\mu k_\mu} \bar{X}_\mu \bar{P}^\top_{-\mu k_{-\mu}}][ P_{\mu k_\mu} X_\mu P^\top_{-\mu k_{-\mu}}+ \bar{P}_{\mu k_\mu} \bar{X}_\mu \bar{P}^\top_{-\mu k_{-\mu}}]^\top P_\mu 
\end{align}
The last step above follows from the eigenvalue decomposition of a cartesian product graph (Lemma \ref{lem:EV}). Note that this only holds if the eigen gap condition (definition \ref{def:sgap}) holds true. 
\begin{align*}
\|\Gamma_\mu\|^2_F  & = \tr(\Gamma^\top_\mu \Gamma_\mu)  \\
& = \tr\Big(\big(P^\top_\mu [ P_{\mu k_\mu} X_\mu P^\top_{-\mu k_{-\mu}}+ \bar{P}_{\mu k_\mu} \bar{X}_\mu \bar{P^\top}_{-\mu k_{-\mu}}][ P_{\mu k_\mu} X_\mu P^\top_{-\mu k_{-\mu}}+ \bar{P}_{\mu k_\mu} \bar{X}_\mu \bar{P}^\top_{-\mu k_{-\mu}}]^\top P_\mu \big)^\top \\
& \times P^\top_\mu [ P_{\mu k_\mu} X_\mu P^\top_{-\mu k_{-\mu}}+ \bar{P}_{\mu k_\mu} \bar{X}_\mu \bar{P}^\top_{-\mu k_{-\mu}}][ P_{\mu k_\mu} X_\mu P^\top_{-\mu k_{-\mu}}+ \bar{P}_{\mu k_\mu} \bar{X}_\mu \bar{P}^\top_{-\mu k_{-\mu}}]^\top P_\mu \Big) \\
& = \tr((X_\mu X^\top_\mu)^\top X_\mu X^\top_\mu) + \tr((\bar{X}_{\mu} \bar{X}^\top_{\mu})^\top \bar{X}_{\mu} \bar{X}^\top_{\mu}) \\
  & = \|X_\mu\|^4_F   +  \| \bar{X}_\mu\|^4_F 
  \end{align*}
 The third step follows from the fact that $P^\top_\mu P_\mu = I$, ${P}^\top_{\mu k_\mu} \bar{P}_{\mu k_\mu} = 0$, ${P}^\top_{-\mu k_{-\mu}} \bar{P}_{-\mu k_{-\mu}} = 0$, ${P}^\top_{\mu k_\mu} {P}_{\mu k_\mu} = I$, $\bar{P}^\top_{\mu k_\mu} \bar{P}_{\mu k_\mu} = I$, ${P}^\top_{-\mu k_{-\mu}} {P}_{-\mu k_{-\mu}} = I$ and  $\bar{P}^\top_{-\mu k_{-\mu}} \bar{P}_{-\mu k_{-\mu}} = I$.

 Let $\{\cdot\}_{1:k_\mu, 1:k_\mu}$ be a matrix operator which represents the selection of the first $k_\mu$ rows and columns of a matrix. Now, expanding the numerator:
 \begin{align*}
 \|\Gamma_\mu(1:k_\mu,1:k_\mu)\|^2_F & =  \|\{P^\top_\mu [ P_{\mu k_\mu} X_\mu X_\mu^\top P^\top_{\mu k_{\mu}}+ \bar{P}_{\mu k_\mu} \bar{X}_\mu \bar{X}^\top_\mu \bar{P}^\top_{\mu k_{\mu}}]P_\mu\}_{1:k_\mu,1:k_\mu}\|^2_F  \\
 & = \|\Big\{\begin{bmatrix} I_{k_\mu \times k_\mu} \\ 0_{(n-k_\mu) \times k_\mu}\end{bmatrix} X_\mu X_\mu^\top {\begin{bmatrix} I_{k_\mu \times k_\mu} \\ 0_{(n-k_\mu) \times k_\mu}\end{bmatrix}}^\top + \begin{bmatrix} 0_{k_\mu \times (n-k_\mu)} \\ I_{(n-k_\mu) \times (n-k_\mu)}\end{bmatrix}\bar{X}_\mu \bar{X}^\top_\mu {\begin{bmatrix} 0_{k_\mu \times (n-k_\mu)} \\ I_{(n-k_\mu) \times (n-k_\mu)}\end{bmatrix}}^\top  \Big\}_{1:k_\mu , 1:k_\mu}\|^2_F \\
 & = \|X_\mu\|^4_F
 \end{align*}
 Finally,
 
  \begin{align*}
  \hat{s}_r(\Gamma^r_\mu) = \frac{\|X_\mu\|^4_F}{ \| X_\mu\|^4_F   +  \| \bar{X}_\mu\|^4_F }
  \end{align*}
From above, $\hat{s}_r(\Gamma^r_\mu) \approx 1$ (Property 2 holds true) if and only if $ \|\bar{X}_\mu\|^2_F  \ll  \| X_\mu\|^2_F$ (Lemma \ref{lem:amlrtg}) and vice versa. 

% Now for correspondence between MLRTG and property 1, note that
% $$\|\diag(\Gamma_\mu)\|^2_F = \|\diag(X_\mu X_\mu^\top)\|^2_F + \|\diag(\bar{X}_\mu \bar{X}_\mu^\top)\|^2_F.$$
%  Also note that the GSC ($\Gamma_\mu$) is a valid covariance matrix, i.e, it is symmetric, positive semi-definite (GSC is a product of a valid covariance with the orthonormal basis). As for MLRTG, $\|\bar{X}_\mu\|_F \ll \|{X}_\mu\|_F$, thus, $\|\diag(\bar{X}_\mu \bar{X}_\mu^\top)\|_F \ll \|\diag(X_\mu X_\mu^\top)\|_F$,
% which implies that the determinant of all the leading $k_\mu$  principal principal minors of the matrix $\Gamma_\mu$ are non-negative. Hence property 1 holds for MLRTG.

\subsubsection{Comments}
The reader might note that the whole framework relies on the existence of the eigen gap condition (definition \ref{def:sgap}). Such a notion does not necessarily exist for the real data, however, we clarify a few important things  right away. The existence of eigen gap is not strict for our framework, thus, practically, it performs reasonably well for a broad range of applications even when such a gap does not exist 2) We introduce this notion to study the theoretical side of our framework and characterize the approximation error (Section \ref{sec:theory}). Experimentally, we have shown (Section \ref{sec:results}) that choosing a specific number of eigenvectors of the graph Laplacians, without knowing the spectral gap is good enough for our setup. Thus, at no point throughout this paper we  presented a way to compute this gap.

%%%%%%%%%%%
\subsection{Approximate MLRTG examples}\label{sec:eg_amlrtg}
 \begin{figure*}[htbp]
    \centering
        \centering
        \includegraphics[width=1.0\textwidth]{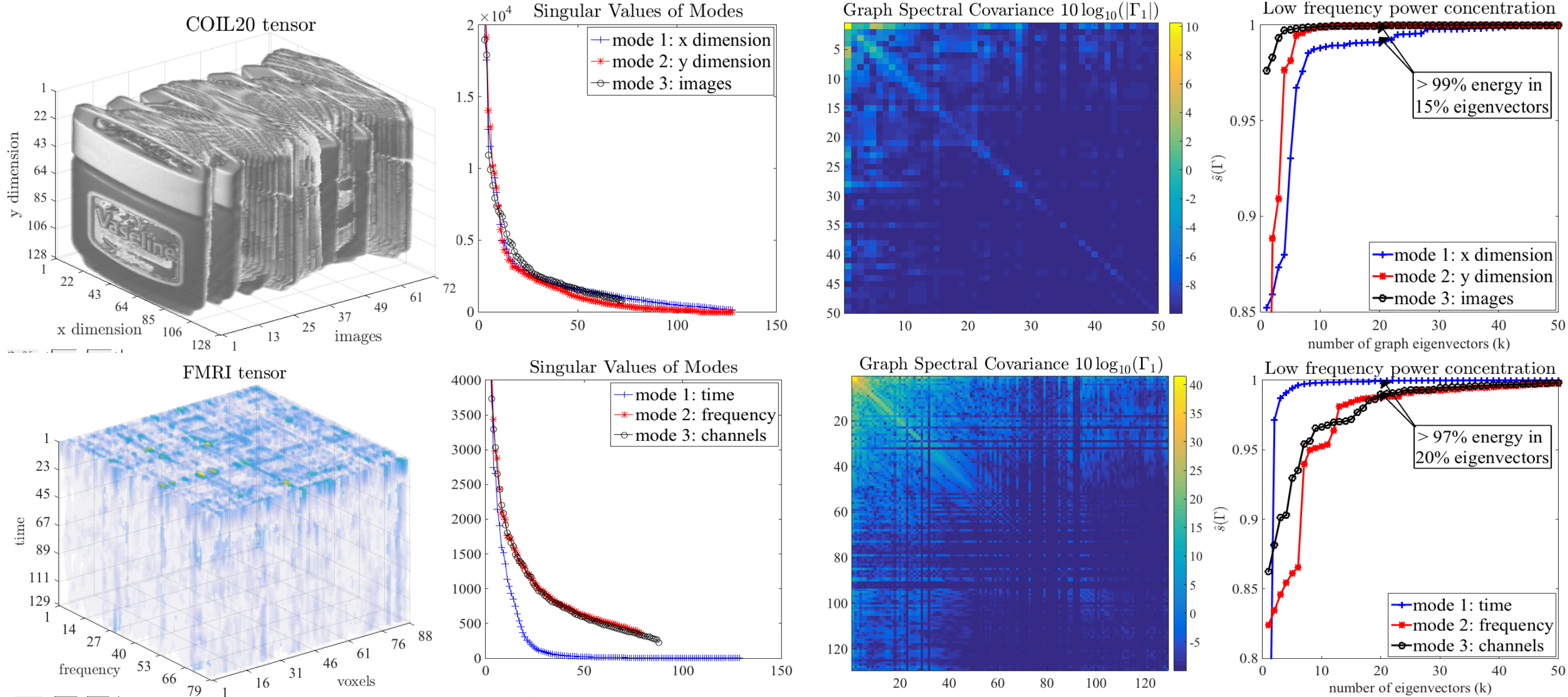}
         \caption{A COIL20 image tensor  and an FMRI tensor. The singular values of the modes, GSC and the energy concentration plot clearly show that the tensor is approximately MLRTG. }
        \label{fig:GSC_two}
    \end{figure*}
%%%%%%%%%%%%

\begin{figure*}[htbp]
    \centering
        \centering
        \includegraphics[width=0.8\textwidth]{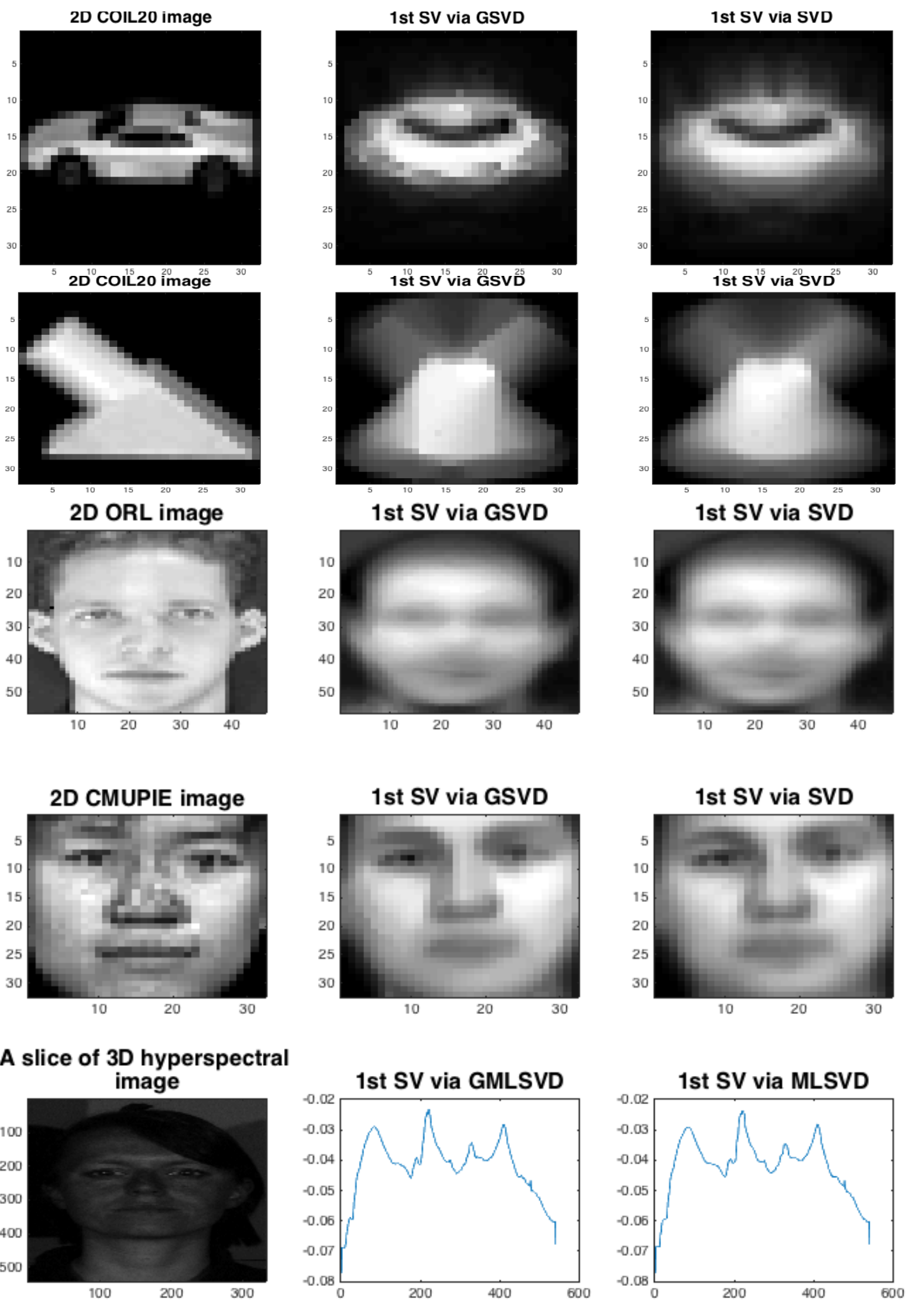}
         \caption{A comparison of the 1st singular vectors obtained via GSVD and SVD for various 2D and 3D real face tensors. In each row, the leftmost plot shows an example image from the database, the middle plot shows the 1st singular vector of mode 1 obtained via GSVD and the right plot shows the 1st singular vector of mode 1 obtained via SVD.  It can be clearly seen that the singular vectors determined by GSVD and SVD are equivalent. }
        \label{fig:eg_algo_gmlsvd}
    \end{figure*}

%%%%%%%%%%%%%

\subsection{Algorithm for TRPCAG} \label{sec:algo_trpca}

We use Parallel Proximal Splitting method   \cite{combettes2011proximal} to  solve this problem. First, we re-write TRPCAG:

\begin{align*}
\min_{\X} \|P_{1k}X_1 P^\top_{2,3 k}-Y_1\|_1 + \gamma \sum_\mu \|X_\mu\|_{*g(\Lambda_{\mu k})}.
\end{align*}

The objective function above can be split as following:
$$f(x) = f_0(x) + \sum_\mu f_\mu(x),$$
where $f_0(x) = \|P_{1k}X_1 P^\top_{2,3 k}-Y_1\|_1$ and $f_\mu (x) = \|X_\mu\|_{*g(\Lambda_{\mu k})}$

To develop a parallel proximal splitting method for the above problem, we need to define the proximal operators of each of the functions $f_\mu(x)$.

Let $\Omega(X,\tau) : \mathbb{R}^{N} \longrightarrow \mathbb{R}^{N}$ denote the element-wise soft-thresholding matrix operator:
$$\Omega(X,\tau) = \sign(X)\max(|X|-\tau,0),$$ 
then we can define 
$$D_\mu(X_\mu,\tau) = A_{1 \mu}\Omega(R_\mu,\tau) A_{2 \mu}^\top$$
as the singular value thresholding operator for matrix $X_\mu$, where $X_\mu = A_\mu R_\mu A_\mu^\top$ is any singular value decomposition of $X_\mu$. Clearly,

$$\prox_{\gamma f_\mu}(X_\mu) = D_\mu(X_\mu, \gamma g(\Lambda_{\mu k})).$$

In order to define the proximal operator of $f_0$ we use the properties of proximal operators from \cite{combettes2011proximal}. Using the fact that $P^\top_{1k} P_{1k} = I$ and $P^\top_{2,3k}P_{2,3k} = I$, we get

$$\prox_{f_0}(X_1) = X_1 + P^\top_{1k}(\Omega(P_{1k}X_1 P^\top_{2,3k}-Y_1,\alpha)-P_{1k}X_1 P^\top_{2,3k}-Y_1)P_{2,3k},$$

where $\alpha$ is the step size. We represent a matrix $U^{a}_b$ to show the $a^{th}$ iteration and matricization along  $b^{th}$ mode. The algorithm can be stated as following: 

\begin{algorithm}
\caption{Parallel Proximal Splitting Algorithm for TRPCAG}
\label{algorithm2}
\begin{algorithmic}
\State INPUT: matricized tensor $Y_1$, weight matrices $g(\Lambda_{\mu k}), \forall \mu$, parameter $\gamma$, $\epsilon \in (0,1)$.
\State $Z_1^{0,0}, \cdots, Z_1^{3,0} \in \Rbb^{k \times k^2}$, all matricized along mode 1.
\State Set $W^0_1 = \sum_{i=0}^3 Z_1^{i,0}$
\For{$j = 1, \cdots, J$}
\State $P^{0,j}_\mu = \prox_{f_0}(W^j_1,\alpha)$
\For{$\mu = 1, \cdots, 3$}
\State $P^{\mu,j} = \prox_{(\gamma f_\mu)}(W^j_\mu,\gamma g(\Lambda_{\mu k}))$
\EndFor
\State $\epsilon \leq \beta_j \leq 2-\epsilon$
\State $P^j_3 = \sum_i P^{i,j}_3$
\For{$l = 0, \cdots, 3$}
\State $Z^{l,j+1}_3 = Z^{l,j}_3 + \beta_j(2P^j_3 - W^j_3 - P^{i,j}_3)$
\EndFor
\State $W^{j+1}_3 = W^j_3 + \beta_j(P^j_3 - W^j_3)$
\State $W^{j+1}_1 = reshape(W^{j+1}_3)$, $Z^{l,j+1}_1 = reshape(Z^{l,j+1}_3)$
\EndFor
\State OUTPUT: $W^{j+1}_1$
\end{algorithmic}
\end{algorithm}

%%%%%%%%%%%%%
\subsection{Proof of Theorem \ref{thm:recovery}}\label{sec:proof_thm_recovery}

\subsubsection{Assumptions}
For the proof of this Theorem, we assume the following, which have already been defined in the proof of Theorem 1:
\begin{itemize}
\item \textbf{Eigen gap assumption:} For a $\K$-nearest neighbors graph constructed from a $k^*$-clusterable data  one can expect $\lambda_{\mu k^*}/\lambda_{\mu k^* + 1} \approx 0$ as $\lambda_{\mu k^*} \approx 0$ and $\lambda_{\mu k^*} \ll \lambda_{\mu k^* + 1}$ (Definition \ref{def:egap}).
\item \textbf{Separable Eigenvector Decomposition of a graph Laplacian:} ${L}_\mu = P_{\mu k} \Lambda_{\mu k} P_{\mu k}^\top = P_{\mu k^*} \Lambda_{\mu k^*} P^{\top}_{\mu k} + \bar{P}_{\mu k^*} \bar{\Lambda}_{\mu k^*} \bar{P}^{\top}_{\mu k^*} $, where $\Lambda_{\mu k^*} \in \Re^{k^* \times k^*}$ is a diagonal matrix of lower eigenvalues  and    $\bar{\Lambda}_{\mu k^*} \in \Re^{(k - k^*) \times (k - k^*)}$ is also a diagonal matrix of higher graph eigenvalues. All values in $\Lambda_\mu$ are sorted in increasing order.
\end{itemize}

\subsubsection{Proof}
Now we prove the 3 parts of this Theorem as following:
\begin{enumerate}
\item The existence of an eigen gap (1st point above) and the definition of MLRTG imply that one can obtain a loss-less compression of the tensor $Y^*$ as:
$$X^* = P_{1k^*}^\top Y P_{2k^*}.$$ 

Note that this compression is just a projection of the rows and columns of $Y^*$ onto the basis vectors which exactly encode the tensor, hence the compression is loss-less. Now, as $X^*$ is loss-less, the singular values of $Y^*$ should be an exact representation of the singular values of $X^*$. Thus, the SVD of $X^* = A_{1k^*}RA_{2k^*}^\top$ implies that $R = S$, where $S$ are the singular values of $Y^*$. Obviously, $V_\mu = P_{\mu k^*} A_{\mu k^*}$ upto a sign permutation because the SVD of $X^*$ is also unique upto a sign permutation (a standard property of SVD).
\item The proof of second part follows directly from that of part 1 (above) and Theorem 1 in \cite{rao2015collaborative}. First, note that for any matrix (2D tensor) $Y \in \Rbb^{n_1 \times n_2}$ eq. \eqref{eq:ginv} can be written as following:
\begin{equation}\label{eq:ginv2}
\min_{X} \phi(M_1(P_{1k}X P^\top_{2k})-Y) + \gamma \|g(\Lambda_{1k})X g(\Lambda_{2k})\|_*,
\end{equation}
where $g(\Lambda_{1k})$ and $g(\Lambda_{2k})$ are diagonal matrices which indicate the weights for the nuclear norm minimization. 

 Let $W_1 = g(\Lambda_{1k})P^\top_{1k}$, $W_2 = P_{2k}g(\Lambda_{2k})$ and $\hat{X} = P_{1k}XP^\top_{2k}$, then we can re-write eq. \eqref{eq:ginv2} as following:
\begin{equation}\label{eq:ginv3}
\min_{\hat{X}} \phi(M(\hat{X})-Y) + \gamma \|W_1 \hat{X} W_2\|_*
\end{equation}

Eq. \eqref{eq:ginv3} is equivalent to the weighted nuclear norm (eq. 11) in \cite{rao2015collaborative}. From the proof of the first part we know that $V_1 = P_{1k}A_{1k}$ and $V_2 = P_{2k} A_{2k}$, thus we can write $\hat{X} = V_1 V^\top_2$. From Theorem 1 in \cite{rao2015collaborative} 
\begin{equation}
\min_{\hat{X}} \|W_1 \hat{X} W_2\|_* = \min_{V_1, V_2} \frac{1}{2}(\|W_1 V_1\|^2_F + \|W_2 V_2\|^2_F), \quad \text{s.t.} \quad \hat{X} = V_1 V^\top_2
\end{equation}

using $g(\tilde{L})_\mu = P_{\mu k} g(\Lambda_{\mu k}) P^\top_{\mu k}$ this is equivalent to the following graph regularized problem:
\begin{align}
\min_{V_1, V_2} & \phi(M(V_1 V^\top_2)-Y) + \gamma_1 \tr(V^\top_1 g(\tilde{L}_1) V_1) + \gamma_2 \tr(V^\top_2 g(\tilde{L}_2) V_2),
\end{align}

\item {To prove the third point we directly work with eq.\eqref{eq:greg} and follow the steps of the proof of Theorem 1 in \cite{shahid2015fast}. We assume the following:
\begin{enumerate}
\item We assume that the observed data matrix $Y$ satisfies $Y = M(Y^*) + E$ where $Y^* \in \MLT$ and $E$ models noise/corruptions. Furthermore, for any $Y^* \in \MLT$ there exists a matrix $C$ such that $Y^* = P_{1k^*} C P_{2k^*}^\top$ and $C = B_{1k^*} B^{\top}_{2k^*}$, so $Y^* = Z^*_1 Z^{*\top}_2$.
\item For the proof of the theorem, we will use the fact that  $V^*_\mu = P_{\mu k^*}A_{\mu k^*} + \bar{P}_{\mu k^*}\bar{A}_{\mu k^*} \in \Re^{n \times k}$, where $P_{\mu k} \in \Rbb^{n \times k^*}$, $\bar{P}_{\mu k^*} \in \Rbb^{n \times (k - k^*)}$ and ${A}_{\mu k^*} \in \Re^{k^* \times k^*}$, $\bar{A}_{\mu k^*} \in \Rbb^{(k-k^*) \times k^*}$. 
\end{enumerate}}

As ${F}^* = V^*_1 V^{*\top}_2$ is the solution of eq.\eqref{eq:greg}, we have

\begin{align}\label{eq:bound1}
& \phi(M({V_1}^*{V_2}^{*\top})-Y) + \gamma_1 \tr({V_1}^{*\top}L_1 {V_1}^*) + \gamma_2 \tr({V_2}^{*\top}L_2 {V_2}^*)  \leq \nonumber \\ 
& \phi(E) + \gamma_1 \tr({Z_1}^{*\top}L_1 {Z_1}^*) + \gamma_2 \tr({Z_2}^{*\top}L_2 {Z_2}^*)
\end{align}

Now using the facts 3b and the eigen gap condition we obtain the following two:

\begin{align}
& \tr({V_1}^{*\top}g(\tilde{L}_1) {V_1}^*) = \tr(A_{1k^*}g(\Lambda_{1k^*})A^{\top}_{1k^*}) + \tr(\bar{A}_{1k^*}g(\bar{\Lambda}_{1k^*})\bar{A}^{\top}_{1k^*}) \nonumber \\
& \geq g(\lambda_{1k^*+1}) \|\bar{A}_{1k^*}\|^2_F = g(\lambda_{1k^*+1}) \|\bar{P}^\top_{1k^*} {V}_1^*\|^2_F
\end{align}
and similarly, 
\begin{equation}
\tr({V_2}^{*\top}g(\tilde{L}_2) {V_2}^*) = g(\lambda_{2k^*+1}) \|\bar{P}^\top_{2k^*} {V}_2^*\|^2_F
\end{equation}

Now, using the fact 3a we get

\begin{equation}
\tr({Z_1}^{*\top}g(\tilde{L}_1) {Z_1}^*) \leq g(\lambda_{1k^*}) \|Z^{*}_1\|^2_F
\end{equation}
and 
\begin{equation}
\tr({Z_2}^{*\top}g(\tilde{L}_2) {Z_2}^*) \leq g(\lambda_{2k^*}) \|Z^{*}_2\|^2_F
\end{equation}

using all the above bounds in eq.\eqref{eq:bound1} yields

$$\phi(F^* - Y) + \gamma_1 g(\lambda_{1k+1}) \|\bar{P}^\top_{1k^*} {V}_1^*\|^2_F + \gamma_2 g(\lambda_{2k^*+1}) \|\bar{P}^\top_{2k^*} {V}_2^*\|^2_F \leq \phi(E) + \gamma \Big(\|Z^*_1\|^2_F \frac{g(\lambda_{1k^*})}{g(\lambda_{1k^*+1})} + \|Z^*_2\|^2_F \frac{g(\lambda_{2k^*})}{g(\lambda_{2k^*+1})} \Big) $$

for our choice of $\gamma_1$ and $\gamma_2$ this yields eq.\eqref{eq:theory}.
\end{enumerate}

\subsection{Experimental Details}\label{sec:experiment_details}
\subsubsection{Generation of artificial datasets}
There are two methods for the generation of artificial datasets:

\textbf{Method 1:} Directly from the eigenvectors of graphs, computed from the modes of a random tensor. 

We describe the procedure for a 2D tensor in detail below:
\begin{enumerate}
\item  Generate a random Gaussian matrix, $Y \in \Rbb^{n \times n}$.
\item Construct a $\K$ graph $\G_1$ between the rows of $Y$ and compute the combinatorial Laplacian $L_1$. 
\item Construct a $\K$ graph $\G_2$ between the rows of $Y^\top$ and compute the combinatorial Laplacian $L_2$.
\item Compute the first $k$ (where $k \ll n$) eigenvectors and eigenvalues of $L_1$ and $L_2$, $(P_{1k}, \Lambda_{1k}), (P_{2k}, \Lambda_{2k})$.
\item Generate a random matrix of the size $X \in \Rbb^{k \times k}$.
\item Generate the low-rank artificial tensor by using $\ve(Y^*) = (P_{1k} \kr P_{2k})\ve(X)$, where $\kr$ denotes the kronecker product.
\end{enumerate}

\textbf{Method 2:} Indirectly, by filtering a randomly generated tensor with the $\K$ combinatorial Laplacians constructed from the flattened  modes of the tensor. We describe the procedure for a 2D tensor below:

\begin{enumerate}
\item Follow steps 1 to 3 of method 1 above.
\item Generate low-rank Laplacians from the eigenvectors: $\hat{L}_{1} = P_{1k} I_{k \times k} P^\top_{1k}$ and $\hat{L}_{2} = P_{2k} I_{k \times k} P^\top_{2k}$.
\item Filter the matrix $Y$ with these Laplacians to get the low-rank matrix $Y^* = \hat{L}_1 Y \hat{L}_2$.
\end{enumerate}

\subsubsection{Information on real datasets}
We report the source, size and dimension of each of the datasets used:

\textbf{2D video datasets:} Three video datasets (900 frames each) collected from the following source: \url{https://sites.google.com/site/backgroundsubtraction/test-sequences}. Each of the frames is vectorized and stacked in the columns of a matrix.

\begin{enumerate}
\item Airport lobby: $25344 \times 900$
\item Shopping Mall: $81920 \times 900$
\item Escalator: $20800 \times 900$
\end{enumerate}

\textbf{3D datasets:} 

\begin{enumerate}
\item Functional Magnetic Resonance Imaging (FMRI): $129 \times 79 \times 1232$ (frequency bins, ROIs, time samples), size 1.5GB.
\item Brain Computer Interface (BCI): $513 \times 256 \times 64$ (time, frequency, channels), size 200MB.
\item Hyperspectral face dataset: $542 \times 333 \times 148$ (y-dimension, x-dimension, spectrum), size 250MB, source: \url{https://scien.stanford.edu/index.php/faces-1-meter-viewing-distance/}.
\item Hyperspectral landscape dataset: $702 \times 1000 \times 148$ (y-dimension, x-dimension, spectrum), size 650MB, source: \url{https://scien.stanford.edu/index.php/landscapes/}.
\item Snowfall video: $1920 \times 1080 \times 500$ (y-dimension, x-dimension, number of frames), size 1.5GB. This video is self made.
\end{enumerate}

\textbf{3D point cloud datasets:} Three 3D datasets ($points \times time \times coordinates$) collected from the following source: \url{http://research.microsoft.com/en-us/um/redmond/events/geometrycompression/data/default.html}. For the purpose of our experiments, we used one of the three coordinates of the tensors only.
\begin{enumerate}
\item dancing man: $1502 \times 573 \times 3$
\item walking dog: $2502 \times 59 \times 3$
\item dancing girl: $2502 \times 720 \times 3$
\end{enumerate}

\textbf{4D dataset:} EEG dataset: $513 \times 128 \times 30 \times 200$, (time, frequency, channels, subjects*trials), size 3GB

\subsubsection{Information on methods and parameters}

\textbf{Robust PCA \cite{candes2011robust}}
\begin{align*}
& \min_X \|X\|_* + \lambda \|S\|_1 \nonumber \\
& \text{s.t.} ~ Y = X + S,
\end{align*}
where $X$ is the low-rank matrix and $S$ is the sparse matrix.

\textbf{Robust PCA on Graphs \cite{shahid2015robust}}

\begin{align*}
& \min_X \|X\|_* + \lambda \|S\|_1 + \gamma \tr(X L_c X^\top) \nonumber \\
& \text{s.t.} ~ Y = X + S,
\end{align*}
where $L_c$ is the combinatorial Laplacian of a $\K$-graph constructed between the columns of Y.

\textbf{Fast Robust PCA on Graphs \cite{shahid2015fast}}

\begin{align*}
\min_X \|Y-X\|_1 + \gamma_c \tr(X L_c X^\top) + \gamma_r \tr(X^\top L_r X),
\end{align*}
where $L_c, L_r$ are the combinatorial Laplacians of $\K$-graphs constructed between the columns and rows of Y.

\textbf{Compressive PCA for low-rank matrices on graphs (CPCA) \cite{shahid2016compressive}}
1) Sample the matrix $Y$ by a factor of $s_r$  along the rows and $s_c$ along the columns  as $\hat{Y} = M_r Y M_c$ and solve FRPCAG to get $\hat{X}$, 2) do the SVD of $\hat{X} = \hat{U}_1 \hat{R} \hat{U}^\top_2$ recovered from step 1, 3) decode the  low-rank $X$ for the full dataset $Y$ by solving the subspace upsampling problems below:

\begin{align*}
& \min_{U_1 \in \Rbb^{n \times k}} \tr(U_1^\top L_r U_1), ~~ \text{s.t.} ~ M_r U_1 = \hat{U}_1 \nonumber \\
& \min_{U_2 \in \Rbb^{n \times k}} \tr(U_2^\top L_c U_2), ~~ \text{s.t.} ~ M_c U_2 = \hat{U}_2
\end{align*}

and then compute $X = U_1 \hat{R} \sqrt{s_r s_c} U^\top_2$.

\begin{table}[htbp]
\centering
\caption{Range of parameter values for each of the models considered in this work.}
\centering
\resizebox{0.6\textwidth}{!}{\begin{tabular}[t]{| c | c | c | } \hline
  \textbf{Model}   & \textbf{Parameters}   & \textbf{Parameter Range} \\\hline
 RPCA \cite{candes2011robust}  & $\lambda$  & $\lambda \in \{\frac{2^{-3}}{\sqrt{\max(n,p)}}:0.1:\frac{2^{3}}{\sqrt{\max(n,p)}}\}$ \\\cline{1-2}
 RPCAG \cite{shahid2015robust} & $\lambda, \gamma$ & $\gamma \in \{2^{-3},2^{-2},\cdots, 2^{10}\}$\\\hline   
  FRPCAG \cite{shahid2015fast} &   $\gamma_{r}, \gamma_{c}$ & $\gamma_{r}, \gamma_{c} \in (0, 30)$    \\\hline
   CPCA &   $\gamma_{r}, \gamma_{c}$ & $\gamma_{r}, \gamma_{c} \in (0, 30)$    \\\cline{2-3}
        &    $k$ (approximate decoder)  & $\hat{R}_{k,k}/\hat{R}_{1,1} < 0.1$  \\\hline
\end{tabular}}\label{tab:models_param}
\end{table}

%   \begin{figure*}[htbp]
%     \centering
%         \centering
%         \includegraphics[width=0.7\textwidth]{gsvd_example.pdf}
%          \caption{Singular value and singular vector study of 2D artificial data corrupted by Gaussian noise. For this experiment, an artificial tensor of rank 10 in every mode is created and then Gaussian noise is added. GSVD and SVD are used to recover the low-rank tensor by using a core size of $20 \times 20$. The first row shows the case of 15dB noise and the second row shows the case of 2dB noise. In both cases note that the singular values of the GSVD recovered tensor (black) are well aligned with those of the clean tensor (blue), while SVD being a simple singular value truncation method, performs worse. Specially for the case of 2dB noise note that the lower singular values are badly affected by noise and this effect is eliminated to some extent by the use of GSVD. From the right plots, one can observe that the first few singular vectors of GSVD are well aligned with those of clean tensor. This alignment is calculated as $|V_1^\top U_1|$, where $V_1$ are the first mode singular vectors obtained via GMLSVD and $U_1$ are the first mode singular vectors of the clean tensor.  }
%         \label{fig:gsvd_example}
%     \end{figure*}

 \begin{sidewaysfigure}[htbp]
        \centering
        \includegraphics[width=0.9\textwidth]{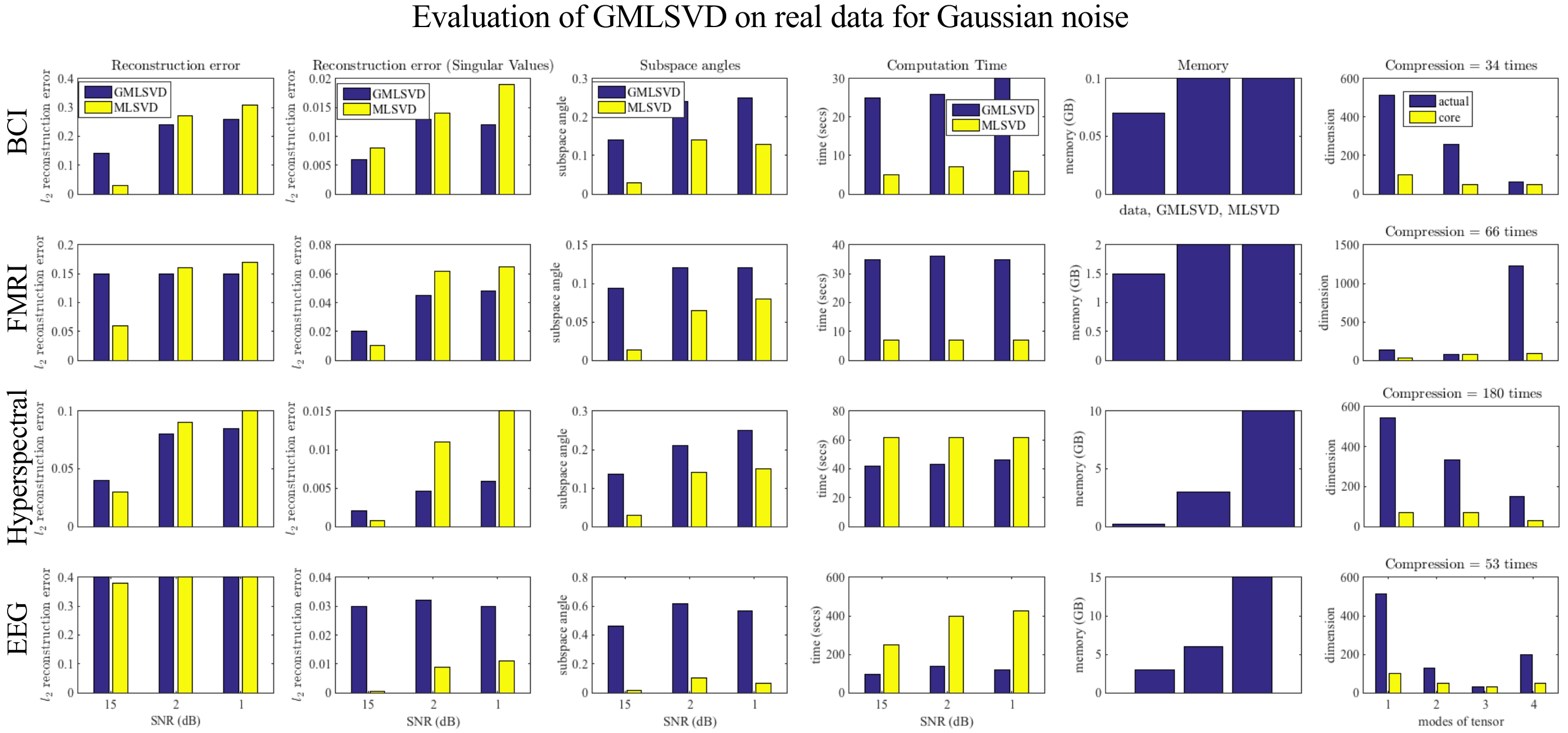}
         \caption{Performance comparison of GMLSVD with MLSVD for four real tensors, Brain Computer Interface (BCI), Functional Magnetic Resonance Imaging (FMRI), Hyperspectral face images and EEG under different SNR scenarios.  The leftmost plots show the $\ell_2$ reconstruction error of the recovered tensor with the clean tensor, the 2nd plots show the reconstruction error for 30 singular values, 3rd plots show the subspace angle between the first 5 singular vectors, 4th plot shows the computation time, 5th plot shows the memory requirement (size of original tensor, memory for GMLSVD, memory for MLSVD) and the rightmost plots show the dimension of the tensor and the core along each mode and hence the amount of compression obtained. Clearly GMLSVD performs better than MLSVD in a low SNR regime. Furthermore, the computation time and memory requirement for GMLSVD is less as compared to MLSVD for big datasets, such as EEG and hyperspectral images. Note that the BCI and FMRI datasets are smaller in size. Thus, even though the time required for GMLSVD is more for these datasets, GMLSVD scales far better than MLSVD.}
        \label{fig:real_data_gaussian_main}
    \end{sidewaysfigure}

    \begin{figure*}[htbp]
    \centering
        \centering
        \includegraphics[width=1.0\textwidth]{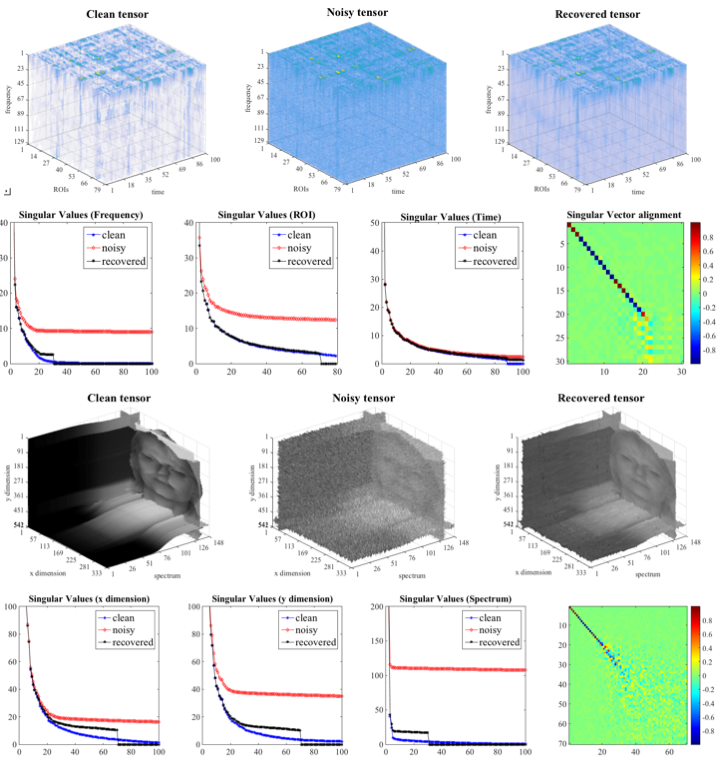}
         \caption{A visualization of the clean, noisy and GMLSVD recovered BCI and hyperspectral face tensors along with a study of the singular values and subspace vectors of various modes. First row shows the BCI tensor and second row shows the singular values recovered for this tensor for various modes. Clearly, the singular values of the recovered tensor (black) are well aligned with those of the clean tensor (blue). Note, how GMLSVD eliminates the effect of noise from the lower singular values. The rightmost plot shows the alignment between the singular vectors of the recovered and clean tensors. This alignment is calculated as $|V_1^\top U_1|$, where $V_1$ are the first mode singular vectors obtained via GMLSVD and $U_1$ are the first mode singular vectors of the clean tensor. Again, note that the first few singular vectors are very well aligned. Third and fourth rows show the same type of results for hyperspectral face tensor.}
        \label{fig:gmlsvd_main}
    \end{figure*}

%  \begin{figure*}[htbp]
%     \centering
%         \centering
%         \includegraphics[width=0.7\textwidth]{hyperspectral_compression_zoomed.pdf}
%          \caption{Qualitative and quantitative results for the face ($542 \times 333 \times 148$) and landscape ($702 \times 1000 \times 148$) 3D hyperspectral tensors. Using cores of size $70 \times 70 \times 30$ for the face and $150 \times 150 \times 30$ for the landscape we attained 150 and 110 times compression while maintaining SNR of 25dB and 15dB.}
%         \label{fig:hyperspectral_compression_zoomed}
%     \end{figure*}
    
     \begin{figure*}[htbp]
    \centering
        \centering
        \includegraphics[width=0.5\textwidth]{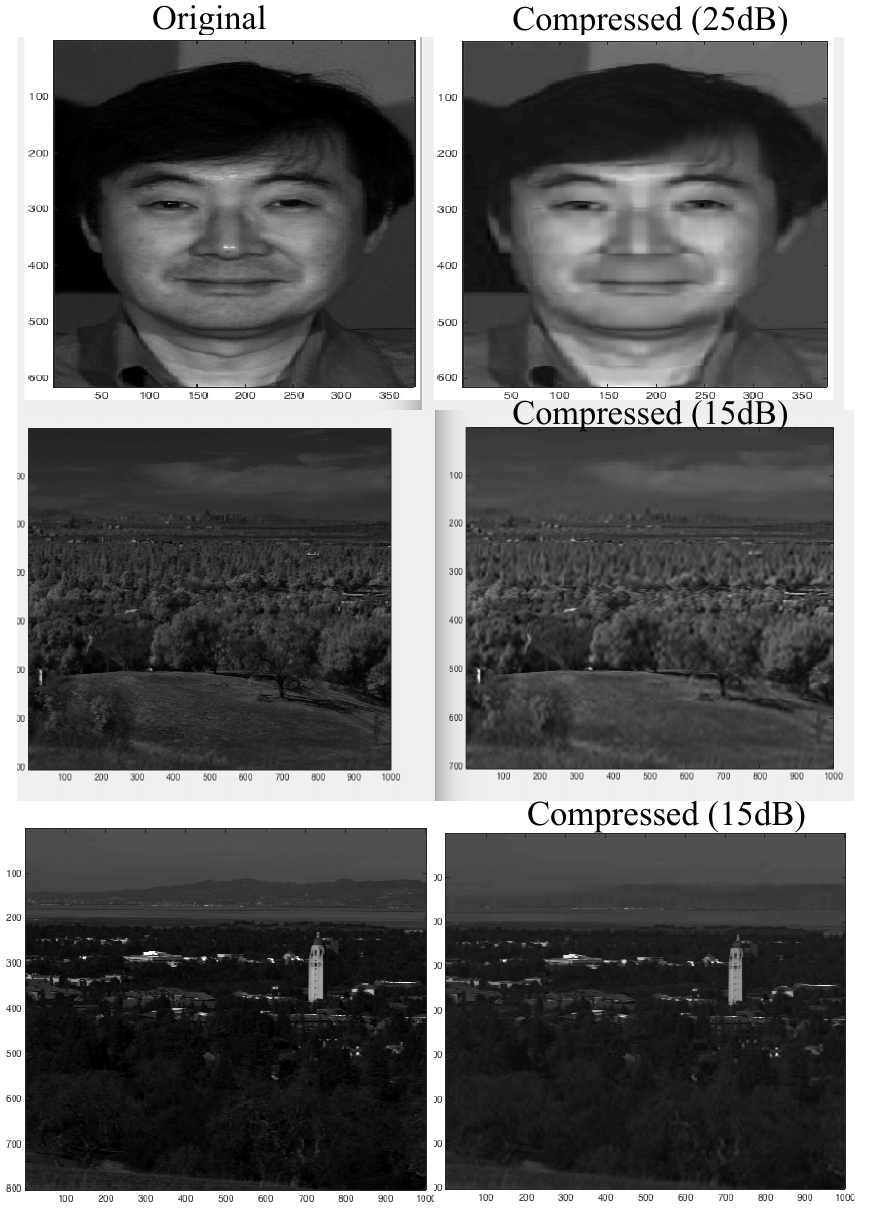}
         \caption{Qualitative and quantitative results for the face ($615 \times 376 \times 148$) and two landscapes ($702 \times 1000 \times 148$ and $801\times 1000 \times 148$) 3D hyperspectral tensors. We used core size of $100 \times 50 \times 20$, $150 \times 150 \times 30$ and $200 \times 200 \times 50$ for the three tensors.}
        \label{fig:hyperspectral_compression_zoomed}
    \end{figure*}
    
\begin{figure}[htbp]
    \centering
        \centering
        \includegraphics[width=0.5\textwidth]{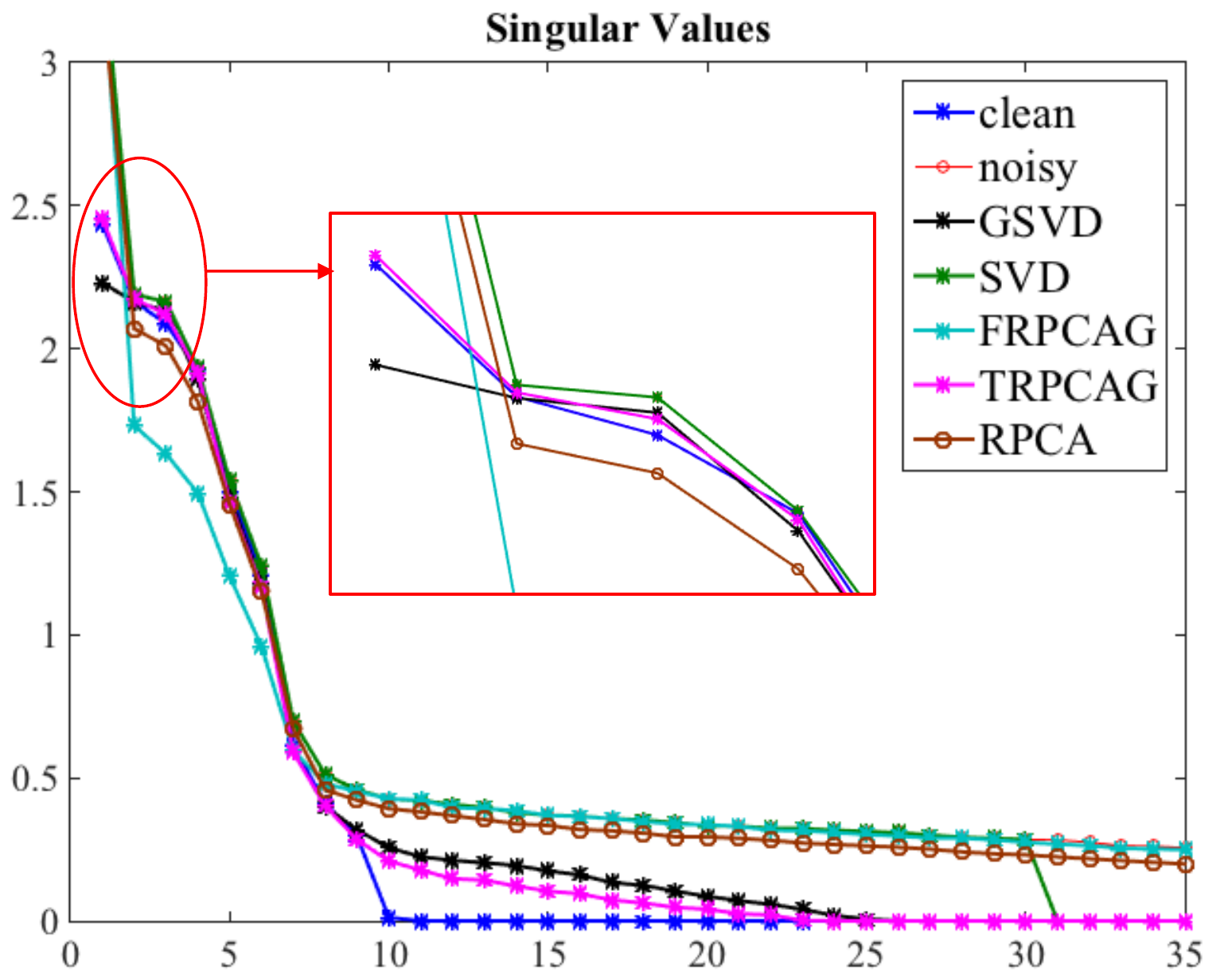}
         \caption{Singular values of GSVD, SVD, FRPCAG, RPCA, TRPCAG and clean data  for a 2D matrix which is corrupted uniformly (10\%) with sparse noise of standard deviation 0.1. Clearly TRPCAG eliminates the effect of noise from lower singular values. In addition note that under sparse noise, the 1st singular value deviates significantly. This effect is also eliminated via TRPCAG.}
        \label{fig:TRPCAG_example}
    \end{figure}
    
     \begin{figure*}[htbp]
    \centering
        \centering
        \includegraphics[width=1.0\textwidth]{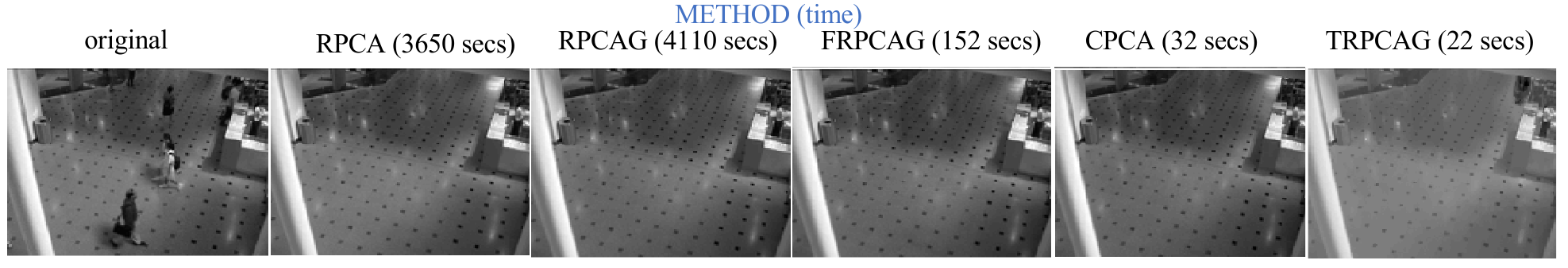}
         \caption{Comparison of the low-rank quality and computation time for TRPCAG, RPCA, RPCAG, FRPCAG and CPCA for a shopping mall lobby surveillance video. TRPCAG recovers a low-rank which is as good as other methods, in a time which is 100 times less as compared to RPCA and RPCAG.}
        \label{fig:video_trpcag_extra}
    \end{figure*}
    
      \begin{figure*}[htbp]
    \centering
        \centering
        \includegraphics[width=0.7\textwidth]{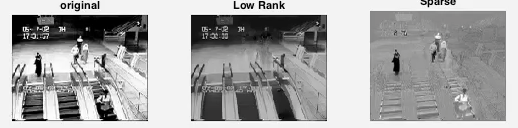}
         \caption{Low-rank and sparse matrix decomposition for an escalator video. left) actual frame from the video, middle) low-rank frame, right) sparse component of the frame.}
        \label{fig:escalator}
    \end{figure*}
    
    \begin{figure*}[htbp]
    \centering
        \centering
        \includegraphics[width=0.6\textwidth]{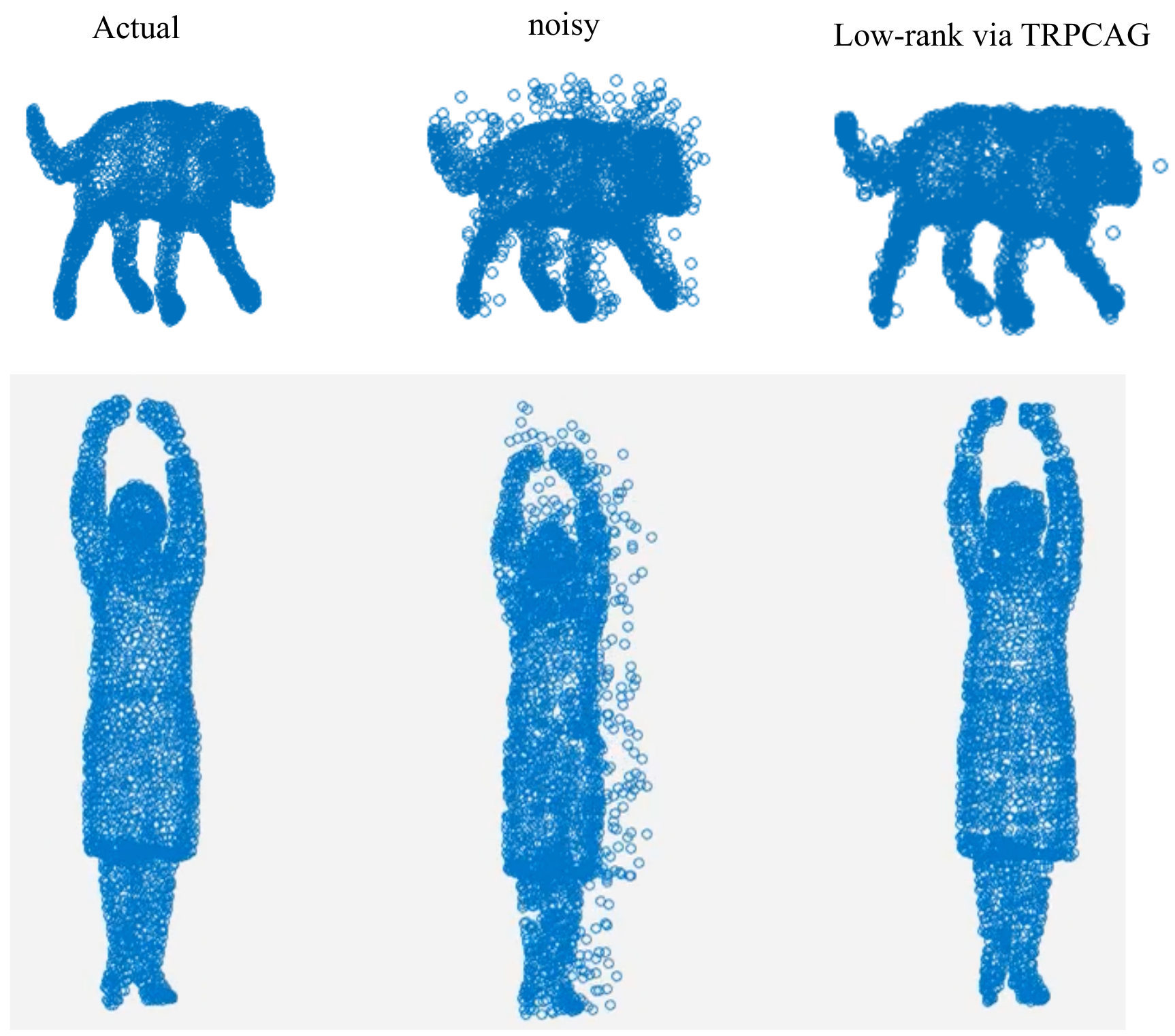}
         \caption{TRPCAG performance for recovering the low-rank point clouds of a walking dog and a dancing girl from the sparse noise. left) actual point cloud, middle) noisy point cloud, right) recovered via TRPCAG.}
        \label{fig:point_cloud}
    \end{figure*}
    
       \begin{figure*}[htbp]
    \centering
        \centering
        \includegraphics[width=0.8\textwidth]{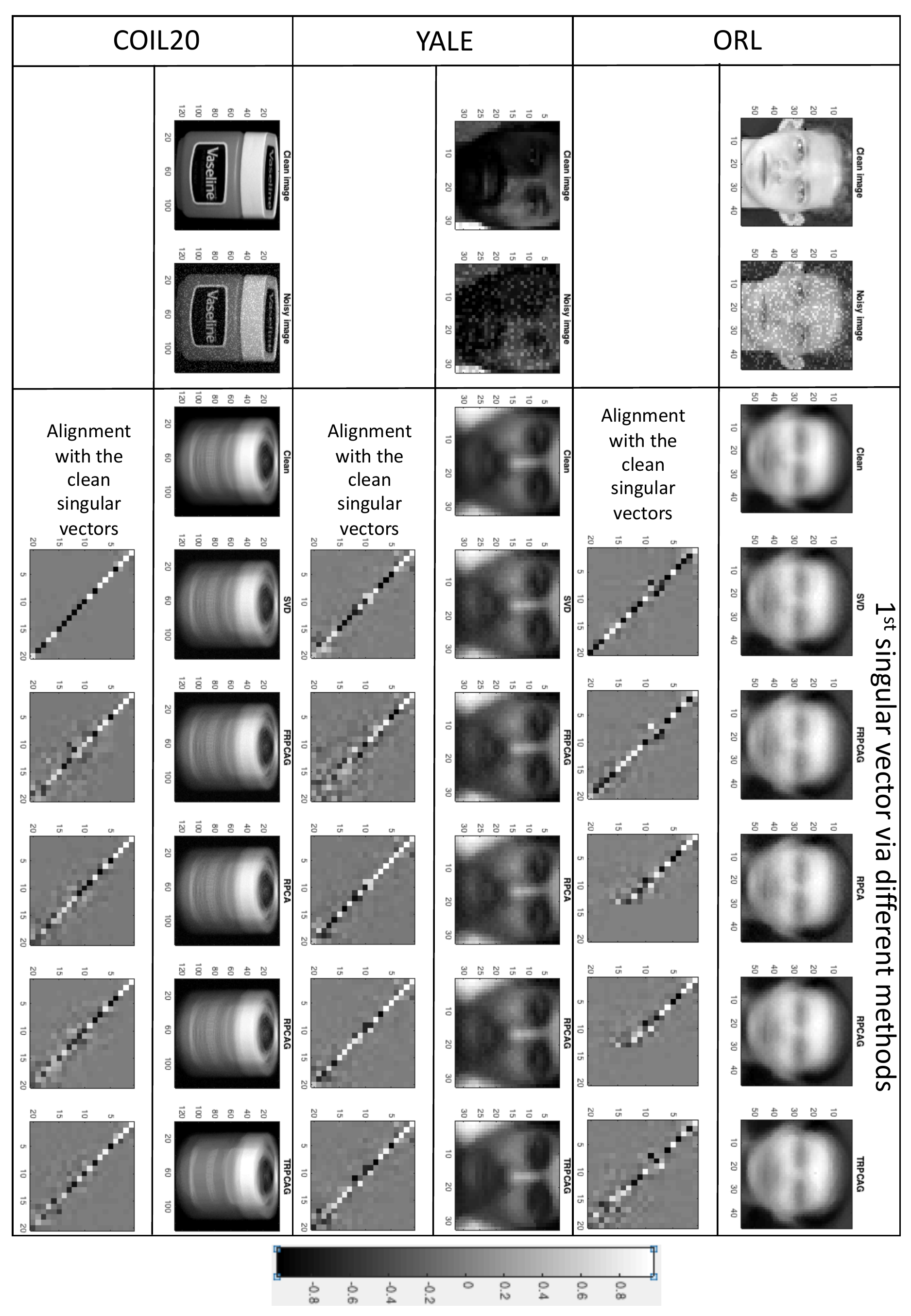}
         \caption{Robust recovery of subspace structures via TRPCAG. The leftmost plots show a clean and sparsely corrupted sample image from various datasets. For each of the datasets, other plots in the 1st row show the 1st singular vector of the clean data and those recovered by various low-rank recovery methods, SVD, FRPCAG, RPCA, RPCAG and TRPCAG and the 2nd row shows the alignment of the 1st 20 singular vecors recovered by these methods with those of the clean tensor. }
        \label{fig:singular_vectors}
    \end{figure*}
    
      \begin{sidewaysfigure}[htbp]
    \centering
        \centering
        \includegraphics[width=1.0\textwidth]{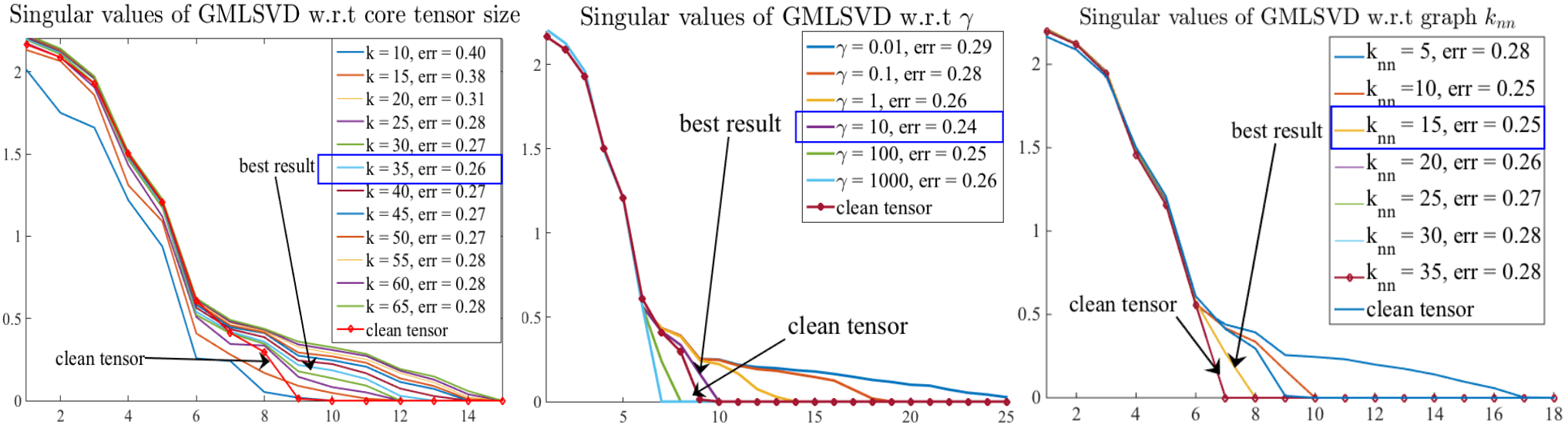}
        \caption{The GCTP (eq.\ref{eq:gctp}) has three parameters: 1) the regularization parameter $\gamma$, 2) the power $\alpha$ in $g(\Lambda_{\mu k}) = \Lambda_{\mu k}^\alpha$ and 3) the graph multilinear rank $k$. Fig.\ref{fig:params} (left plot) demonstrates the effect of varying the multilinear rank $k$ on the singular values and the reconstruction error. For this experiment, we fix $\gamma = 1$ and use a $100 \times 100$ matrix of rank $10$ and add Gaussian noise to attain a SNR of 5dB. The best error occurs at $k  = 35 (> 10)$ as the eigen gap condition does not hold for the graphs and the dataset is noisy. Note also that the error decreases when $k$ is increased from 10 to 35 and then starts increasing again. This observation motivates the use of a $k > k^*$ ($k^* = 10$ here) in the absence of an eigen gap, that supports our theoretical finding in Section \ref{sec:theory}. Next, we fix $k = 35$ and repeat the experiments for $\gamma$. The best result (middle plot) occurs at $\gamma = 10$. Finally, we fix $k = 35, \gamma = 10$ and run experiments for different values of $\alpha$. No change in the error and singular values was observed for different values of $\alpha$. This shows that if the rank $k$ and $\gamma$ are tuned, the optimization problem is not sensitive to $\alpha$. Although the graph $\K$ parameter is not a model parameter of GCTP, it still effects the quality of results to some extent, as shown in the rightmost plot of Fig.\ref{fig:params}. For this experiment we fix $\gamma = 10, k = 35, \alpha = 1$ and use the same dataset with 5dB SNR. The best result occurs for $\K = 15$ which is not surprising, as this value results in a sparse graph which keeps the most important connections only. Furthermore, note that the singular values of the recovered tensor for $\K = 15$ are closest to those of the clean tensor.}
        %  \caption{Effect of parameters $k, \gamma, \K$ on the quality of singular values and reconstruction error for a $100 \times 100$ matrix of rank 10 via GMLSVD. Clearly, the best results occur at $k = 35, \gamma = 10$ and $\K = 15$. }
        \label{fig:params}
        %\vspace{-1.2em}
    \end{sidewaysfigure}

\end{document}